\documentclass[sigconf]{acmart}

\AtBeginDocument{%
  }

\setcopyright{acmlicensed}
\copyrightyear{2018}
\acmYear{2018}
\acmDOI{XXXXXXX.XXXXXXX}

\acmConference[Conference acronym 'XX]{Make sure to enter the correct
  conference title from your rights confirmation email}{June 03--05,
  2018}{Woodstock, NY}

\acmISBN{978-1-4503-XXXX-X/2018/06}

\usepackage[textsize=tiny]{todonotes}
\usepackage{xcolor}
\usepackage{threeparttable}
\newtheorem{theorem}{Theorem}
\newtheorem{lemma}{Lemma}
\usepackage{enumitem}

\usepackage{amsmath}

\usepackage{amssymb}
\usepackage{mathtools}
\usepackage{amsthm}
\usepackage{wrapfig}

\usepackage{enumitem}

\usepackage{hyperref}
\usepackage{makecell}
\usepackage{multirow}

\usepackage{microtype}
\usepackage{graphicx}
\usepackage{subfigure}
\usepackage{booktabs} 
\usepackage{multicol}
\usepackage{algorithm}
\usepackage{algpseudocode}
\usepackage{appendix}
\usepackage{cuted}
\usepackage{wrapfig}

\usepackage{mathtools}

\usepackage{amssymb}
\usepackage{amsthm}
\usepackage{color}

\newcommand{\cut}[1]{{}}

\newcommand{\vA}{{\mathbf{A}}}

\newcommand{\cE}{{\mathcal{E}}}

\newcommand{\cG}{{\mathcal{G}}}

\newcommand{\cN}{{\mathcal{N}}}

\newcommand{\cV}{{\mathcal{V}}}

\newcommand{\RR}{\mathbb{R}}

\makeatletter
\let\@@span\span
\def\sp@n{\@@span\omit\advance\@multicnt\m@ne}
\makeatother

\newcommand{\bc}{\begin{center}}
\newcommand{\ec}{\end{center}}

\newcommand{\bdm}{\begin{displaymath}}
\newcommand{\edm}{\end{displaymath}}

\newcommand{\beq}{\begin{equation}}
\newcommand{\eeq}{\end{equation}}

\newcommand{\bfl}{\begin{flushleft}}
\newcommand{\efl}{\end{flushleft}}

\newcommand{\bt}{\begin{tabbing}}
\newcommand{\et}{\end{tabbing}}

\newcommand{\beqn}{\begin{align}}
\newcommand{\eeqn}{\end{align}}

\newcommand{\beqs}{\begin{align*}} 
\newcommand{\eeqs}{\end{align*}}  

\begin{document}

\title{VISAGNN: Versatile Staleness-Aware Efficient Training on Large-Scale Graphs}

\author{Rui Xue}
\affiliation{%
  \institution{North Carolina State University}
  \city{Raleigh}
  \country{USA}
}

\begin{abstract}
Graph Neural Networks (GNNs) have shown exceptional success in graph representation learning and a wide range of real-world applications. However, scaling deeper GNNs poses challenges due to the neighbor explosion problem when training on large-scale graphs. To mitigate this, a promising class of GNN training algorithms utilizes historical embeddings to reduce computation and memory costs while preserving the expressiveness of the model. These methods leverage historical embeddings for out-of-batch nodes, effectively approximating full-batch training without losing any neighbor information—a limitation found in traditional sampling methods. However, the staleness of these historical embeddings often introduces significant bias, acting as a bottleneck that can adversely affect model performance. In this paper, we propose a novel \underline{V}ersat\underline{I}le \underline{S}taleness-\underline{A}ware GNN, named \textit{VISAGNN}, which dynamically and adaptively incorporates staleness criteria into the large-scale GNN training process. By embedding staleness into the message-passing mechanism, loss function, and historical embeddings during training, our approach enables the model to adaptively mitigate the negative effects of stale embeddings, thereby reducing estimation errors and enhancing downstream accuracy. Comprehensive experiments demonstrate the effectiveness of our method in overcoming the staleness issue of existing historical embedding techniques, showcasing its superior performance and efficiency on large-scale benchmarks, along with significantly faster convergence.
\end{abstract}

\maketitle

\section{Introduction}
\label{sec:intro}

Graph Neural Networks (GNNs) have proven to be highly effective tools for learning representations from graph-structured data~\citep{hamilton2020graph,ma2021deep}, excelling in tasks such as node classification, link prediction, and graph classification~\citep{kipf2016semi,gasteiger2018combining,velivckovic2017graph,wu2019simplifying, xue2023lazygnn, xue2023efficient}. They have also been successfully applied in real-world scenarios like recommendation systems, biological molecule modeling, and transportation networks~\citep{tang2020knowing, sankar2021graph,fout2017protein,wu2022graph, zhang2024linear}. However, the scalability of GNNs is challenged by their recursive message-passing process, which results in the neighborhood explosion problem. This issue arises because the number of neighbors involved in mini-batch computations grows exponentially with the number of GNN layers~\citep{hamilton2017inductive, chen2018fastgcn, han2023mlpinit}, making it difficult for deeper GNNs to capture long-range dependencies on large graphs. Such long-range information is known to enhance GNN performance~\citep{gasteiger2018predict, gu2020implicit, liu2020towards, chen2020simple, li2021training, ma2020unified, pan2020_unified, zhu2021interpreting, chen2020graph}, but the neighborhood explosion problem limits the ability of GNNs to handle large-scale graphs within the constraints of GPU memory and computational resources during training and inference. This bottleneck significantly hampers the expressive power of GNNs and their applicability to large-scale graphs.

Various approaches have been developed to enhance the scalability of GNNs, including sampling techniques~\citep{hamilton2017inductive,chen2018fastgcn,chiang2019cluster,Zeng2020GraphSAINT, xue2024haste}, pre- and post-computing strategies~\citep{wu2019simplifying, rossi2020sign, sun2021scalable, huang2020combining}, and distributed learning~\citep{chai2022distributed, shao2022distributed}. Among these, sampling methods are widely used to address the neighborhood explosion problem in large-scale GNNs due to their simplicity and promising results. However, sampling methods often discard information from unsampled neighbors during training, and because nodes in a graph are interconnected and cannot simply be treated as independent and identically distributed (\emph{iid}), this leads to estimation variance in embedding approximation and an inevitable loss of accurate graph information.

To address this issue, historical embedding methods have been proposed, such as VR-GCN~\citep{chen2017stochastic}, MVS-GCN~\citep{cong2020minimal}, GAS~\citep{fey2021gnnautoscale}, GraphFM~\citep{yu2022graphfm} and Refresh~\citep{huang2023refresh}. These methods use historical embeddings of unsampled neighbors as approximations of their true aggregated embeddings. During each training iteration, they store intermediate node embeddings at each GNN layer as historical embeddings, which are then utilized in subsequent iterations. This approach effectively mitigates the neighbor explosion problem and reduces the variance associated with sampling methods by preserving all neighbor information. The historical embeddings can be stored offline on CPU memory or disk, conserving GPU memory. These approaches avoid ignoring any nodes or edges, thereby reducing variance and maintaining the expressiveness of the backbone GNNs while achieving strong scalability and efficiency.

While using historical embeddings can provide several benefits, their quality is a crucial determinant of overall performance. The discrepancy between a true node embedding and its corresponding historical embedding, which we refer to as the staleness of the historical embeddings, becomes a critical factor since the historical embedding serves as an approximation of the true one. This phenomenon is particularly evident in large-scale datasets, as the update speed of historical embeddings lags far behind that of model parameters. Specifically, each node’s cache is refreshed only once per epoch when it serves as a target node, whereas model parameters are updated $\frac{N}{B}$ times, where $N$ and $B$ denote the number of nodes and batch size, respectively. The gap widens as $N$ increases or $B$ decreases. As a result, these historical embeddings become highly stale and exhibit significant discrepancies from the true embeddings. Consequently, historical embedding methods often suffer from substantial degradation in both prediction accuracy and convergence speed when compared to vanilla sampling methods like GraphSAGE~\citep{hamilton2017inductive}, which do not rely on historical embeddings. Thus, staleness becomes the primary bottleneck for these methods ~\cite{huang2023refresh}.

Motivated by our findings and analysis, effectively utilizing staleness to leverage fresh embeddings while minimizing the impact of stale embeddings has become a critical issue. Although some existing works attempt to reduce staleness by evicting or implicitly compensating for the negative impact of stale embeddings, they cannot fully eliminate staleness and often introduce additional bias. To address this limitation, we propose a \underline{V}ersat\underline{I}le \underline{S}taleness-\underline{A}ware GNN (\textit{VISAGNN}), which \textbf{explicitly} incorporates staleness criteria through three key components:
(1) \textbf{Dynamic Staleness Attention}: We introduce a novel staleness-based weighted message-passing mechanism that uses staleness scores as a metric to dynamically determine the importance of each node during message passing;
(2) \textbf{Staleness-aware Loss}: We design a regularization term based on staleness criterion to be included in the loss function, explicitly reducing the influence of staleness on the model; 
(3) \textbf{Staleness-Augmented Embeddings}: We offer a straightforward solution by directly injecting staleness into the node embeddings. These three innovative techniques effectively mitigate the staleness issues present in all related approaches. Our proposed framework is highly flexible, orthogonal, and compatible with various sampling methods, historical embedding techniques, and system-level optimization strategies. Comprehensive experiments demonstrate that it further enhances existing historical embedding methods by improving performance, accelerating convergence, and maintaining high efficiency.

\section{Related Work}
\label{sec:related}

In this section, we summarize related works on the scalability of large-scale GNNs with a focus on sampling methods.

\noindent \textbf{Sampling methods.} 
Sampling methods utilize mini-batch training strategies by selecting a subgraph as a small batch, reducing computation and memory requirements. These methods fall into three main categories: (1)\textit{Node-wise sampling} samples a fixed number of neighbors per hop, as seen in models like GraphSAGE~\citep{hamilton2017inductive}, PinSAGE~\citep{ying2018graph}, and GraphFM-IB ~\citep{yu2022graphfm}. However, because it involves dropping unsampled nodes and edges, it introduces bias and variance. Additionally, while it helps mitigate the neighbor explosion problem, it doesn't entirely solve it since the number of neighbors still grows exponentially. 
(2)\textit{Layer-wise sampling} addresses the neighbor explosion problem by fixing the number of sampled neighbors per layer. For example, FastGCN~\citep{chen2018fastgcn} treats message passing from an integral perspective, independently sampling nodes in each GNN layer using importance sampling. LADIES and ASGCN~\citep{zou2019layer, huang2018adaptive} incorporates inter-layer correlations during sampling process. However, the adjacency matrix generated by layer-wise sampling tends to be sparser than that of other methods, often leading to suboptimal performance.
(3) \textit{Subgraph sampling} samples subgraphs as mini-batches and performs message passing within them, mitigating the neighbor explosion problem. ClusterGCN~\citep{chiang2019cluster} achieves this by clustering the graph and forming each mini-batch from several clusters. GraphSaint~\citep{Zeng2020GraphSAINT} extends this with diverse samplers and importance sampling to reduce bias and variance. However, these methods can still suffer from high variance due to the ignored edges between subgraphs.

\noindent \textbf{Historical embedding methods.}
While sampling methods effectively alleviate the neighbor explosion problem, they often suffer from performance degradation due to the variance introduced by dropping nodes and edges. To address this issue, some approaches have started using historical embeddings as an approximation for the true embeddings obtained from full-batch computation. This allows them to avoid dropping any nodes or edges while still reducing memory costs by limiting the number of sampled neighbors. VR-GCN\citep{chen2017stochastic} was the first to propose using historical embeddings for out-of-batch nodes to reduce variance. MVS-GCN\citep{cong2020minimal} improved this approach with a one-shot sampling strategy, eliminating the need for nodes to recursively explore their neighborhoods in each layer. GNNAutoScale~\citep{fey2021gnnautoscale} restricts the receptive field to direct one-hop neighbors, enabling constant GPU memory consumption while still preserving all relevant neighbor information. LMC~\citep{shi2023lmc} considered backward propagation, retrieving discarded embeddings during backward passes, which improved performance and accelerated convergence.

Although these historical embedding approaches are promising due to their strong performance and scalability, they are limited by approximation errors caused by the staleness of the historical embeddings. This issue becomes more pronounced with large-scale datasets. To address this, GAS~\citep{fey2021gnnautoscale} mitigates staleness by using graph clustering to reduce inter-connectivity and applying regularization to limit parameter changes, thus reducing approximation errors. GraphFM-OB\citep{yu2022graphfm} compensates for staleness by leveraging feature momentum for in-batch nodes nd out-of-batch nodes. Despite these efforts, these methods only tackle the issue superficially, resulting in minimal performance improvements. Refresh~\citep{huang2023refresh} introduces a staleness score, which quantifies the degree of staleness, and avoids using stale embeddings to alleviate this issue. However, it results in the loss of some direct neighbor information, introducing significant bias.

\section{Methodology}
\label{sec:method}
In this section, we first mathematically formulate historical embedding methods and then theoretically demonstrate that staleness is a key factor in the effectiveness of these methods. Then, we present a novel VISAGNN, that dynamically and explicitly incorporates staleness into the training process from multiple sensory perspectives, utilizing the staleness criterion as a metric to prioritize fresher embeddings over stale ones. The overall workflow is shown in Figure~\ref{fig:alg}. Here, we first present the definition:

\noindent \textbf{Notations.}
A graph is represented by $\cG = (\cV, \cE)$ where $\cV = \{v_1, \dots, v_n\}$ is the set of $n$ nodes and $\mathcal{E} = \{e_1, \dots, e_m\}$ is the set of $m$ edges.
The graph structure of $\cG$ can be represented by an adjacency matrix $\vA \in \RR^{n\times n}$, where $\vA_{ij}>0$ when there exists an edge between node $v_i$ and $v_j$, and $\vA_{i,j}=0$ otherwise. The neighboring nodes of node $v$ is denoted by $\cN(v)$.
The symmetrically normalized graph Laplacian matrix is defined as $L=I-\hat{A}$ with $\hat{A}=D^{-1/2}AD^{-1/2}$ where $D$ is the degree matrix. We denote by $\bar{h}_i^{(l)}$ the historical embedding of node $i$ at layer $l$, and by $\tilde{h}_i^{(l)}$ the approximate embedding of the exact embedding $h_i^{(l)}$, computed using the historical embeddings $\bar{h}^{(l)}$, with the staleness $s_i$ representing the approximation error. Formally, we define the staleness at layer $l$ for node $i$ as
$s_i = \;\big\lVert \bar{h}_i^{(l)} - h_i^{(l)}\big\rVert,$
which quantifies the discrepancy between the historical embedding $\bar{h}_i^{(l)}$ and the true embedding $h_i^{(l)}$.

\subsection{Staleness of historical embedding methods}
Sampling is used to generate mini-batches for message passing to address the scalability challenge in large-scale graphs:
\begin{align}
h_i^{(l+1)} &= g_\theta^{(l+1)}\big(h_i^{l}, [h_j^{l}]_{j\in \cN(i)}\big) \approx g_\theta^{(l+1)}\big(h_i^l, [h_j^l]_{j \in \mathcal{N}(i) \cap B}\big)
\end{align}

\noindent Here, \( h_i^l \) represent the feature embedding of the in-batch node \( i \) at the \( l \)-th layer, and \( g_\theta^{(l+1)} \) denote the message-passing update function at the $l+1$-th layer with parameters \( \theta \). The set \( \mathcal{N}(i) \cap B \) refers to the sampled 1-hop neighborhood of node \( i \) in current batch $B$. However, the large variance arises because the out-of-batch neighbors $[h_j^{l}]_{j\in \mathcal{N}(i)\setminus B}$ are not considered during aggregation.
 
To address this issue, historical embedding methods utilize historical embeddings 
\([ \bar{h}_j^l ]_{j \in \mathcal{N}(i) \setminus B}\) to approximate the embeddings of out-of-batch neighbor nodes \([ h_j^l ]_{j \in \mathcal{N}(i) \setminus B}\) at each layer, providing an approximation of full-batch aggregation. The feature memory is then updated for future use, using only the in-batch node embeddings \(\bar{h}_i^{l+1} = h_i^{l+1}\). This process can be expressed as:

\begin{align}
h_i^{(l+1)} &= g_\theta^{(l+1)}\big(h_i^{l}, [h_j^{l}]_{j\in \cN(i)}\big)\\
&= g_\theta^{(l+1)}\big(h_i^{l}, \underbrace{[h_j^{l}]_{j\in \cN(i)\cap B}}_{\text{in-batch neighbors}}\cup\underbrace{[h_j^{l}]_{j\in \mathcal{N}(i)\setminus B}}_{\text{out-of-batch neighbors}}\big)\\
&\approx g_\theta^{(l+1)}\big(h_i^{l}, \underbrace{[h_j^{l}]_{j\in \cN(i)\cap B}}_{\text{in-batch neighbors}}\cup\underbrace{[\Bar{h}_j^{l}]_{j\in \mathcal{N}(i)\setminus B}\big)}_{\text{historical embeddings}},
\end{align}

\noindent While using historical embeddings as approximations helps retain information for out-of-batch nodes and ensures constant memory usage \citep{fey2021gnnautoscale}, large approximation errors in certain stale embeddings can significantly degrade model performance. To highlight this issue and motivate our approach, we first present a theoretical analysis showing that the approximation error of the final embeddings is upper bounded by the staleness. Our analysis adheres to the assumptions outlined in all related works.

\begin{theorem}[Embeddings Approximation Error] 
\label{thm:approximation}
Assuming a L-layers GNN \( g_\theta^{(l)}(h) \) with a Lipschitz constant \( \beta^{(l)} \) for each layer \( l = 1, \dots, L \), and \( \mathcal{N}(i) \) is the set of neighbor nodes of \( i \), \( \forall i \in V \). \( \|\bar{h}^{(l)} - h^{(l)}\| \) represents the distance between the historical embeddings and the true embeddings, which corresponds to the staleness. The approximation error of the final layer embeddings $\Tilde{h}^{(L)}_i$ is then upper bounded by:

\begin{equation*}
||\Tilde{h}^{(L)}_i - h^{(L)}_i||\leq \sum_{k=1}^{L}\big(\prod_{l=k+1}^{L}\beta^{(l)}  |\cN(i)|*||\Tilde{\hat{A}}_{i,}||*||\Bar{h}^{(k-1)}-h^{(k-1)}||\big).\\
\end{equation*}
\end{theorem}

\noindent The detailed proof is provided in the Appendix \ref{app:proof1}. From the above theorem, we can observe that the distance between the final layer's embeddings produced by historical embedding methods and full aggregations is bounded by a cumulative sum of the per-layer approximation error \( ||\Bar{h}^{(k-1)} - h^{(k-1)}|| \). To prevent the accumulation of staleness across layers from having a significantly negative impact on the quality of the final embeddings, reducing the impact of staleness at each layer becomes a crucial issue. 

From Theorem \ref{thm:approximation}, $||\Bar{h}^{(k-1)}-h^{(k-1)}||$ directly measures the staleness which is the distance between historical embeddings and true embeddings. However, it is impractical to recompute the true embedding $h^{(k-1)}$ at every iteration due to the significantly higher computational overhead involved. Hence, inspired by existing work \citep{huang2023refresh}, we adopt two indicators to represent the staleness criterion $s_i$: the persistence time $T_i$ and the gradient norm $||\nabla L_{\theta}(h_i)||$, both of which have been shown to be effective. We cache these two indicators from each layer along with the corresponding historical embeddings for use in our training framework. We detail these two designs as follows:

\begin{itemize} [leftmargin=*, itemsep=0pt, topsep=0pt]
    \item The persistence $T_i$ for a specific node \(i\) measures how many training iterations the historical embedding remains unchanged before being updated again. Since a specific node is sampled as a target only once per epoch, and historical embedding methods update the cached embeddings of target nodes only at every iteration, the embedding of that node remains unchanged for the rest of the iterations. However, the model parameters continue to update throughout all training iterations. \textbf{This is the the root cause of staleness.}  $T_i$ reflects the gap between the update frequencies of the historical embeddings and the model parameters, directly capturing staleness in a straightforward manner. A high persistence value indicates that the historical embedding has not been updated recently, leading to stronger feature staleness. 

    \item The norm of the gradient metric $||\nabla L_{\theta}(h_i)||$ reflects the extent of changes in the model parameters, which can also indicate the staleness. A small gradient magnitude suggests that the model parameters are not changing significantly, leading to stable node embeddings throughout the training iterations. Consequently, the approximation error between the historical embeddings and the exact embedding is likely to be small, resulting in minimal staleness.

\end{itemize}

\begin{figure*}[t]
    \begin{minipage}[c]{0.7\textwidth}
    \centering
    \includegraphics[width=1.0\linewidth]{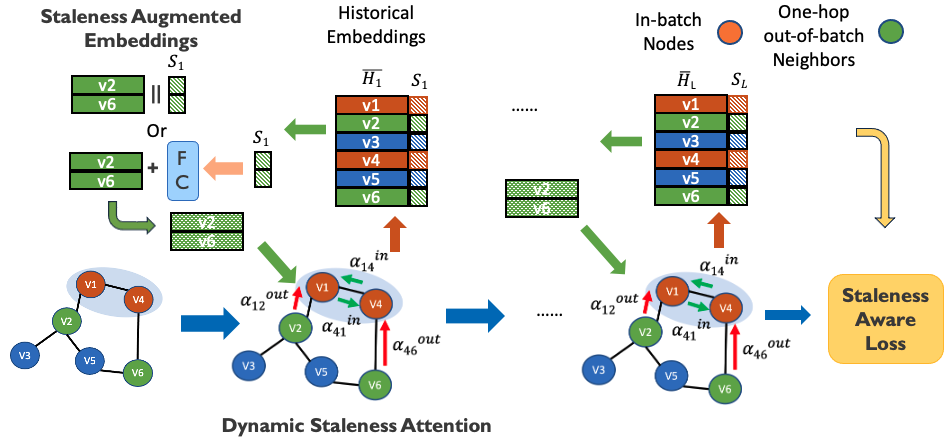}
    \end{minipage}\hfill
    \begin{minipage}[c]{0.25\textwidth}
    \caption{Three key designs in VISAGNN. (1) \textbf{Augmented Embeddings}: VISAGNN offers two ways to integrate staleness  criterion into historical embeddings. (2) \textbf{Dynamic Staleness Attention}: VISAGNN performs weighted message passing based on both feature embeddings and staleness  criterion. (3) \textbf{Staleness-aware loss}: A regularization term based on staleness is incorporated into the loss function in VISAGNN.
    }\label{fig:alg}
    \end{minipage}
\end{figure*}

\subsection{Dynamic Staleness Attention}
As introduced in Sec. \ref{sec:intro}, staleness becomes a bottleneck for existing historical embedding methods. While existing works like Refresh \citep{huang2023refresh} utilize staleness criteria as thresholds to evict embeddings that have not been recently updated or are unstable, simply discarding these embeddings based on staleness can introduce significant bias. This approach also makes the model overly sensitive to the fixed staleness threshold, as the embeddings of any nodes with staleness exceeding the threshold are discarded, even though their degrees of staleness may vary. Consequently, this motivates us to seek a better scheme that preserves neighbors but \textbf{dynamically re-weights their messages} according to the current degree of staleness and graph structure.

We propose a staleness-aware attention mechanism by incorporating the staleness criterion into the graph attention formulation. This mechanism dynamically adjusts attention coefficients based on the node’s current features, staleness, graph structure, and training progress. The attention coefficients for in-batch neighbors \(\alpha_{ij}^{\text{in}}\) and out-of-batch neighbors \(\alpha_{ij}^{\text{out}}\) between node \(i\) and node \(j\) at epoch \(t\) are formulated as follows. We omit the layer number $L$ in $\alpha$ for simplicity.

\noindent where $\mathbf{W}$ is the weight matrix applied for each nodes for feature transformation, $\mathbf{a}$ is a learnable weight vector used to compute the attention score between two nodes, similar as GAT. The operator $\parallel$ denotes concatenation. The term $s_j$ represents the staleness criterion of node $j$, reflecting how outdated the embedding of node $j$ is. The function $\sigma$ represents the nonlinear function, we specifically utilize sigmoid function in this paper, defined as $f(x) = \frac{1}{1+e^{-x}}$. The value $c_j$ is a centrality measure for node $j$, while $c_{avg}$ denotes the average centrality measure across the graph. In this paper, we use node degree as centrality metric to evaluate the importance of each node. The time-dependent coefficient $\gamma(t) = \frac{\boldsymbol{\beta}}{t}$, where $t$ is the current epoch, $\boldsymbol{\beta}$ is a learnable scaling factor that controls how quickly \(\gamma(t)\) decreases with the training process for each node. It modulates the impact of the staleness on attention during the training process. $\phi$ is a non-linear activation function. Note that $\boldsymbol{\alpha_{ij}^{\text{in}}}(t)$ degenerates into the traditional attention scores in GAT when staleness equals 0, which aligns with our intuition.

\vspace{8mm}
\begin{strip}
  \makebox[\textwidth]{
    \begin{minipage}{0.9\textwidth}
      \begin{equation}
        \boldsymbol{\alpha_{ij}^{\text{out}}}(t) = \frac{\exp\left(\text{LeakyReLU}\left(\mathbf{a}^T \left[\mathbf{W}h_i \parallel \mathbf{W}\Bar{h}_j\right] \right) - \boldsymbol{\gamma}(t) \cdot s_j \cdot \sigma(c_j - c_{\text{avg}})\right)}{\sum_{k \in \mathcal{N}(i) \setminus B} \exp\left(\text{LeakyReLU}\left(\mathbf{a}^T \left[\mathbf{W}h_i \parallel \mathbf{W}\Bar{h}_k\right] \right) - \boldsymbol{\gamma}(t) \cdot  s_k \cdot \sigma(c_k - c_{\text{avg}})\right)}
        \label{eqn:attn}
      \end{equation}
      \begin{equation}
        \boldsymbol{\alpha_{ij}^{\text{in}}}(t) = \left. \boldsymbol{\alpha_{ij}^{\text{out}}}(t) \right|_{s_j, s_k = 0, \Bar{h} = \Tilde{h}}
      \end{equation}
      \begin{equation}
        \Tilde{h}^{(L)}_i = \phi\left( \sum_{j\in \cN(i)\cap B} \boldsymbol{\alpha_{ij}}^{L-1,\text{in}} \mathbf{W} h^{(L-1)}_j, \sum_{j\in \mathcal{N}(i)\setminus B} \boldsymbol{\alpha_{ij}}^{L-1,\text{out}} \mathbf{W} \Bar{h}^{(L-1)}_j \right)
      \end{equation}
    \end{minipage}
  }
  \vspace{-2mm}
\end{strip}

The core of our design revolves around the term $- \gamma(t) * s_j * \sigma(c_j - c_{avg})$, which consists of three components: 

\noindent ~\textbf{(1) Staleness Criterion:} $s_j$ represents the staleness of each node embedding, which is the key for achieving staleness-aware attention. In our implementation, we choose to use the gradient criterion $||\nabla L_{\theta}(h_i)||$, which has shown to be more effective and stable than the persistence $T_i$ in our experiment. The gradients at any layer are obtained from backward propagation when the corresponding node was included into the computation graph previously.

\noindent ~\textbf{(2) Centrality:} After considering the impact of feature embeddings, centrality $c$ is introduced to incorporate the graph structure. We use node degree as centrality $c$ in this work. The motivation is that if a stale node is important, the negative effects caused by staleness will be amplified. Specifically, when the degree of a node is high and the staleness is also high, this term significantly penalizes and reduces the attention coefficient to mitigate the impact of staleness, as these stale embeddings are propagated through many neighboring nodes. Conversely, if the staleness is low, the node’s embedding is fresh and will not cause significant negative effects, allowing it to be effectively utilized. When the degree is low, these nodes are less critical to the final representation, so staleness may have a smaller impact. Furthermore, we choose to use relative centrality by subtracting the average node degree of the graph from each node’s degree, \(c_j - c_{\text{avg}}\), to prevent the issue that graphs with dense connections naturally have high node degrees. We then use the sigmoid function to further reduce the scale impact.

\noindent ~\textbf{(3) Decay Coefficient:} We also introduce a function $\gamma(t)$ related to the training process as a coefficient for the staleness term. The reason for this is that as training progresses, the model parameters gradually converge, leading to minimal updates of the embeddings in the final few epochs. Therefore, the influence of staleness should not play a significant role when calculating the attention score. Although there are many feasible designs, we directly used $\frac{\boldsymbol{\beta}}{t}$ for the sake of simplicity, where $\boldsymbol{\beta}$ is a learnable parameter.

\subsection{Staleness-aware Loss}
While the preceding components mitigate the negative impact of staleness in the forward pass, their influence on how stale embeddings propagate during optimization is relatively limited. To ensure the model consistently prioritizes fresh information, we introduce an explicit staleness criterion into the loss function, guiding optimization to control staleness in the backward pass. Specifically, we incorporate staleness into the optimization process as a regularization term to more effectively mitigate its effects. However, the gradient criterion $||\nabla L_{\theta}(h_i)||$ for staleness is not feasible to use since the loss has not yet been computed. From Theorem \ref{thm:approximation}, we find that the final representation contains the accumulated staleness from all layers, allowing it to effectively represent staleness. Hence, we choose to utilize the feature embeddings of in-batch nodes at last layer between two consecutive epochs for our design. The staleness-aware regularization term is defined as follows:
\begin{equation}
\mathcal{L}_{\text{stale}} = \sum_{i\in B } ||h_{i,k}^{(L)} - h_{i,k-1}^{(L)} ||^2
\end{equation}
\noindent where $h_{i,k}^{(L)}$ represents the feature embedding of node \( i \) at the final layer $L$ during epoch $k$. This design is based on the observation that as training progresses, model parameters tend to converge, resulting in smaller gradient values and fewer updates to the embeddings in later epochs. Consequently, the difference between final representations from consecutive epochs becomes progressively smaller, particularly after the model has been trained for several epochs. This aligns with our earlier conclusion that the influence of staleness diminishes as training progresses. Another advantage of this design is that it does not introduce any additional computational overhead.

By jointly optimizing both the downstream tasks and the staleness issue, the gradient also becomes staleness-aware, which better mitigates the negative effects of staleness on the model's performance. For the sake of simplicity, we define the overall training loss as follows:
\begin{equation}
\mathcal{L}(\theta) = \mathcal{L}_{\text{task}}(\theta) + \lambda \cdot\mathcal{L}_{\text{stale}}(\theta)
\label{eqn:loss}
\end{equation}

\noindent where $\lambda$ is a hyperparameter that controls the degree to which staleness affects loss and gradient computation. $\mathcal{L}_{\text{task}}$ is the task-specific loss, such as cross-entropy in node classification.

We provide a theorem that guarantees the convergence of VISAGNN when incorporating dynamic staleness attention along with the staleness-aware loss, as stated below:

\begin{theorem}
Consider VISAGNN using the staleness attention from Eqn.~\ref{eqn:attn} to minimize the loss function $L(\theta)$ defined in Eqn.~\ref{eqn:loss}, which is $L$-smooth with respect to parameters $\theta$. Suppose the stochastic gradient $g_t$ satisfies $\mathbb E[g_t|\theta_t]=\nabla L(\theta_t)$ and $\text{Var}[g_t|\theta_t]\le\sigma^2$, $L^* = \min_\theta L(\theta)$ and learning rate $\eta < \frac{2}{L}$. The algorithm converges and satisfies:
\begin{equation}
\frac{1}{T}\sum_{t=0}^{T-1} \mathbb{E}[\|\nabla L(\theta_t)\|^2] \leq \frac{2(L(\theta_0) - L^*)}{\eta T(2-\eta L)} + \frac{\eta L \sigma^2}{2-\eta L}
\end{equation}
\end{theorem}
\noindent The proof can be found in Appendix~\ref{app:proof2}.

\subsection{Staleness-Augmented Embeddings}

While the proposed attention and loss functions can make the model aware of staleness, they do not explicitly convey the actual staleness level of each embedding. The dynamic attention may down-weight stale information, and the loss function may penalize it, but neither provides a direct signal about how stale an embedding is. To address this, we inject a staleness channel directly into the historical embeddings. This explicit staleness encoding is integrated using one of two methods: \textit{Concatenation} or \textit{Summation}.

\noindent \textbf{Concatenation}: We treat staleness as an additional dimension of feature and concatenate it with the historical embeddings at each layer. To ensure that the impact of staleness is appropriately balanced, we first normalize the staleness criterion using a log normalization technique. This approach helps to mitigate the influence of imbalanced distributions, such as extremely high staleness criterion values, ensuring that stale embeddings do not dominate the feature representation. It also prevents staleness from being overly influential due to differences in scale when combined with the node features. The augmented embeddings can be represented as:
\begin{equation}
    \Bar{h_j}' = \text{Concat}\left(\Bar{h_j}, log(1 + s_j)\right)
\end{equation}

\noindent \textbf{Summation}: This approach differs from simple concatenation. We combine the staleness criterion with the node features through a non-linear transformation, allowing the model to learn a weighted combination of the node’s inherent features and its staleness, potentially capturing their interactions and enhancing the expressiveness of the learned node representations. Specifically, suppose \(\mathbf{W}_s\) is a learnable weight matrix, the transformation can be represented as:
\begin{equation}
    \Bar{h_j}' = \Bar{h_j} + \phi(\mathbf{W}_s \cdot s_j)
\end{equation}

Similarly, we use $||\nabla L_{\theta}(h_i)||$ as the staleness criterion $s_i$. This choice is based on our experiments, which indicate that extreme values in the persistence time $T_i$ can adversely affect aggregation, causing the model to overly focus on the staleness term. Consequently, this negatively affects convergence, especially on very large datasets. However, we still utilize $T_i$ as used in Refresh~\cite{huang2023refresh}: We set a dataset-dependent high-value threshold $G_{\text{thres}}$ and drop a small portion of nodes whose persistence times $T_i$ significantly exceed it in each training iteration to mitigate extreme staleness, while leaving the majority of nodes unaffected.

\subsection{Complexity Analysis}
\label{sec:complexity}
Despite incorporating staleness-aware mechanisms, VISAGNN retains computational complexity on par with that of standard GAT. This is because VISAGNN only performs local neighbor attention—not global self-attention, which usually incurs $\mathcal{O}(n^2)$ complexity—meaning attention is computed only over existing edges. Specifically, the complexity remains $\mathcal{O}(nFF' + mF')$, where $F$ and $F'$ are the input and output feature dimensions, $n$ and $m$ are the numbers of nodes and edges, respectively. The additional staleness terms $\gamma(t) \cdot s_j \cdot \sigma(c_j - c_{avg})$ only introduce $\mathcal{O}(1)$ scalar operations per edge, resulting in $\mathcal{O}(m)$ additional cost, which is negligible compared to the $F'$-dimensional vector operations.

\section{Experiments}
\label{sec:exp}

\subsection{Performance}
\textbf{Experimental setting.} 
We present a performance comparison against major baselines, including several classical GNN models such as GCN~\citep{kipf2016semi} and SGC~\citep{wu2019simplifying}, sampling based methods, such as GraphSAGE~\citep{hamilton2017inductive}, FastGCN~\citep{chen2018fastgcn}, LADIES~\citep{zou2019layer}, Cluster-GCN~\citep{chiang2019cluster}, and GraphSAINT~\citep{Zeng2020GraphSAINT}. Additionally, we include state-of-the-art historical embeddings methods such as VR-GCN~\citep{chen2017stochastic}, MVS-GCN~\citep{cong2020minimal}, GNNAutoScale (GAS)~\citep{fey2021gnnautoscale}, GraphFM~\citep{yu2022graphfm}, Refresh~\citep{huang2023refresh} and LMC~\citep{shi2023lmc}. For the last four models, we employ GAT as the GNN backbone to ensure a fair comparison with our proposed methods. We conduct experiments on five widely used large-scale graph datasets: REDDIT, OGBN-Arxiv, OGBN-Products, OGBN-Papers100M, and MAG240M~\citep{hu2020open}, which span different domains and vary in scale, with the last two datasets considered as large-scale benchmarks. We report the best performance each model can achieve in Table \ref{tab:baseline}, following the configurations provided in their respective papers and official repositories. We present the performance comparison under various batch sizes in Table \ref{tab:BS}, as batch size is a key factor influencing the extent of staleness, as discussed in Sec. \ref{sec:intro}. For cases where no performance metrics or specific hyperparameters were provided in the original or baseline papers~\citep{fey2021gnnautoscale}, we use “—” to remain consistent with the baseline papers. For results in Table 2, due to the significant computational overhead required for these two datasets, we perform our own hyperparameter fine-tuning only on the historical embedding methods. For SGC, we use the results reported on the ogb-leaderboard. GraphSAGE runs out of memory (OOM) on these two large-scale datasets. We denote the augmentation strategies of concatenation and summation introduced in Section \ref{sec:method} as VISAGNN-Cat and VISAGNN-Sum, respectively.
The performance results are reported in Table~\ref{tab:baseline} and ~\ref{tab:papers100m}.

VISAGNN's hyperparameters are tuned from the following search space: (1) learning rate: $\{0.01, 0.001, 0.0001\}$;  (2) weight decay: $\{0, 5e-4, 5e-5\}$; (3) dropout: $\{0.1, 0.3, 0.5, 0.7\}$; (4) propagation layers : $L \in \{1, 2, 3\}$; (5) MLP hidden units: $\{256, 512\}$; (6) $\lambda \in \{0.1, 0.3, 0.5, 0.8\}$.

\begin{table}
\caption{Accuracy comparison ($\%$) with major baselines. “—” indicates that no reported performance/ hyperparameters.}
\vspace{-2mm}
\label{tab:baseline}
\resizebox{1\linewidth}{!}{%
\begin{tabular}{lccccc}
\toprule
    &{\footnotesize{\textbf{\#\,nodes}}} & \footnotesize{230K} & \footnotesize{169K} & \footnotesize{2.4M} \\ [-0.1cm] &
    {\footnotesize{\textbf{\#\,edges}}} & \footnotesize{11.6M} & \footnotesize{1.2M} & \footnotesize{61.9M} \\ [-0.05cm] \multirow{2}{*}{\textbf{Method}} & \multirow{2}{*}{\textbf{GNNs}} & \multirow{2}{*}{\textsc{Reddit}} & \texttt{ogbn} & \texttt{ogbn}\\ 
    & & & \texttt{arxiv} & \texttt{products} \\
\midrule
\multirow{8}{*}{\textbf{Scalable}} & GraphSAGE    & 95.4$\pm$0.1  & 71.5$\pm$0.2 & 78.7$\pm$0.1 \\
& FastGCN    & 93.7$\pm$0.1  & --- & ---  \\
& LADIES    & 92.8$\pm$0.1  & --- & --- \\
& Cluster-GCN    & 96.6$\pm$0.1  & --- & 79.0$\pm$0.2  \\
& GraphSAINT    & \textbf{97.0$\pm$0.1}  & --- & 79.1$\pm$0.3  \\
& SGC    & 96.4$\pm$0.1  & --- & ---  \\
\hline
\multirow{3}{*}{\textbf{Full Batch}} 
& GCN    & 95.4$\pm$0.1  & 71.6$\pm$0.1 & OOM   \\
& GAT   & 95.7$\pm$0.1 & 71.5$\pm$0.2 &  OOM\\
& APPNP   & 96.1$\pm$0.2  & 71.8$\pm$0.1 & OOM\\
\hline
\multirow{6}{*}{\textbf{Historical}} 
& VR-GCN    & 94.1$\pm$0.2  & 71.5$\pm$0.1 & 76.3$\pm$0.3 \\
& MVS-GNN  & 94.9$\pm$0.1 & 71.6$\pm$0.1 & 76.9$\pm$0.1 \\
&\multirow{1}{*}{GAS}
 & 95.7$\pm$0.1 & 71.7$\pm$0.2 & 77.0$\pm$0.3 \\
&\multirow{1}{*}{GraphFM}
& 95.6$\pm$0.2 & 71.9$\pm$0.2 & 77.2$\pm$0.2 \\
&\multirow{1}{*}{Refresh}
& 95.4$\pm$0.2 & 70.4$\pm$0.2 & 78.7$\pm$0.2 \\
&\multirow{1}{*}{LMC}
& 96.2$\pm$0.1 & 72.2$\pm$0.1 &  77.5$\pm$0.3 \\
\hline
\multirow{2}{*}{\textbf{Ours}} & \multirow{1}{*}{\textbf{VISAGNN-Cat}} & 96.5$\pm$0.2 & 73.0$\pm$0.1 & 79.9$\pm$0.3 \\
&\multirow{1}{*}{\textbf{VISAGNN-Sum}} & 96.6$\pm$0.2 & \textbf{73.2$\pm$0.2} & \textbf{80.2$\pm$0.2} \\
\bottomrule
\end{tabular}
}
\vspace{-5mm}
\end{table} 

\textbf{Performance analysis.} 
From the results of the performance comparison, we can draw the following observations:

\noindent ~$\bullet$ \textbf{VISAGNN is specifically designed to address the staleness issues present in existing models, delivering robust and consistently better solutions for large-scale datasets across both sampling and historical embedding methods, while also accelerating convergence—rather than merely achieving large performance gains on a few datasets.} Notably, staleness becomes a significant problem only on large-scale datasets. Consequently, our model exhibits \textbf{substantially greater improvements on datasets such as ogbn-products, ogbn-papers100M, and MAG240M}, while achieving relatively slight improvements on smaller datasets. This clearly demonstrates our model's strong capability to address the staleness issue.

\noindent ~$\bullet$ None of the existing historical embedding methods consistently outperform classical models on large-scale datasets such as ogbn-products, ogbn-papers100m and MAG240M, and they only surpass other scalable methods by a small margin on other datasets. For example, SGC performs almost the best among these baselines on two large scale datasets, ogbn-papers100m and MAG240M without relying on any complex model architecture design. This is due to the slower update of historical embeddings compared to model parameters, especially given the large number of batches in a single training epoch, highlighting staleness as a significant bottleneck for all historical embedding techniques. It is also worth noting that while Refresh performs well on large-scale datasets, it falls significantly behind other baselines when staleness is not dominant (ogbn-arxiv). This is because it simply evicts some important neighbors, which can potentially introduce significant bias, reinforcing our claim made in Section \ref{sec:related}.

\noindent ~$\bullet$ On large-scale datasets, the proposed VISAGNN outperforms all baselines on ogbn-arxiv, ogbn-products, ogbn-papers100M, and MAG-240M, particularly compared to state-of-the-art historical embedding methods (e.g., \textbf{+2.1\% on ogbn-papers100M and +2.4\% on MAG-240M}). Notably, VISAGNN shows substantial improvements on large scale datasets, highlighting the necessity and significance of the staleness-aware techniques we introduced, especially under conditions of increased staleness. Furthermore, VISAGNN-Sum surpasses VISAGNN-Cat, indicating that using a learnable fully connected layer is more effective for integrating staleness information into node embeddings, resulting in improved final representations.

\noindent ~$\bullet$ The strategies we proposed in VISAGNN can be integrated with various baselines. For instance, in LMC, historical gradients also encounter the issue of staleness, which dynamic attention can help alleviate during gradient message passing. This advantage underscores the flexibility and adaptability of our model.

\begin{table}
\centering
\renewcommand{\arraystretch}{1.2}
\caption{Prediction accuracy (\%) comparison with other baselines on \textbf{ogbn-papers100M} and MAG240M}
\vspace{-2mm}
\label{tab:papers100m}
\resizebox{\linewidth}{!}{%
\begin{tabular}{l|ccccccc}
\toprule
\textbf{Method} & \textbf{Sage} & \textbf{SGC} & \textbf{GAS} & \textbf{FM} & \textbf{Refresh} & \textbf{LMC}  & \textbf{VISAGNN}\\
\midrule
papers100M & OOM & 63.3 & 57.5 & 58.6 & 65.4 & 61.3 & \textbf{67.5}\\
MAG240M & OOM & 65.3 & 61.2 & 62.3 & 64.8 & 62.5 & \textbf{67.2}\\
\bottomrule
\end{tabular}
}
\vspace{-5mm}
\end{table}

\begin{table*}[th]
\centering
\caption{Memory usage (MB) and running time (seconds) on arxiv and products.}
\vspace{-3mm}
\label{tab:efficiency}
\setlength{\tabcolsep}{2pt}
\resizebox{0.8\linewidth}{!}{%
\begin{tabular}{c|ccccc|ccccc} 
\toprule
\multirow{3}{*}{\textbf{Dataset}} & \multicolumn{5}{c}{\textsc{Memory (MB)}} & \multicolumn{5}{c}{\textsc{Time (s)}} \\
\cmidrule(lr){2-6} \cmidrule(lr){7-11}
& \multirow{2}{*}{\textbf{Sage}} & \multirow{2}{*}{\textbf{GAS}} & \multirow{2}{*}{\textbf{Refresh}} & \textbf{VISAGNN} & \textbf{VISAGNN} & \multirow{2}{*}{\textbf{SAGE}} & \multirow{2}{*}{\textbf{GAS}} & \multirow{2}{*}{\textbf{Refresh}} & \textbf{VISAGNN} & \textbf{VISAGNN}\\
& & & & \textbf{-Cat} & \textbf{-Sum} & & & &\textbf{-Cat} & \textbf{-Sum}\\
\midrule
ogbn-arxiv & 2997 & 767 & 791 & 813 & 869 & 21 & 40 & 49 & 22 & 26\\
\midrule
ogbn-products & OOM & 8886 & 8933 & 8982 & 9017 & N/A & 2522 & 2178 & 1303 & 1380\\
\bottomrule
\end{tabular}
}
\end{table*}

\begin{table}[th]
\centering
\renewcommand{\arraystretch}{1}
\caption{Impact of different components}
\vspace{-3mm}
\label{components}
\resizebox{0.85\linewidth}{!}{%
\begin{tabular}{l|cc}
\toprule
\textbf{Method} & \textbf{ogbn-arxiv} & \textbf{ogbn-products}\\
\midrule
VISAGNN w/o att & 72.0$\pm$0.3 & 77.8$\pm$0.2 \\
VISAGNN w/o loss & 72.6$\pm$0.3 & 79.2$\pm$0.1 \\
VISAGNN w/o emb & 72.9$\pm$0.2 & 79.8$\pm$0.2 \\
VISAGNN & 73.2$\pm$0.2 & 80.2$\pm$0.2 \\
\bottomrule
\end{tabular}
}
\vspace{-3mm} 
\end{table}

\subsection{Efficiency Analysis}
\label{sec:efficiency} 
In this section, we present an efficiency analysis—reporting memory usage and total runtime at each model’s best performance. We chose ogbn-arxiv and ogbn-products datasets because the baselines supply their optimal hyperparameters on them. We compare our method against one classical scalable GNN (GraphSAGE) and two historical embedding techniques (GAS and Refresh), as summarized in Table~\ref{tab:efficiency}. To ensure a fair comparison, we employed the official implementations for all baseline methods and kept the hyperparameters consistent. For GAS and Refresh, we used GAT as the GNN backbone since both methods also leverage attention mechanisms.

From the results, we observe that GraphSAGE still suffers from the neighbor explosion problem, leading to out-of-memory (OOM) errors on ogbn-products and significantly higher memory costs for ogbn-arxiv in our experiments. Refresh requires less running time on ogbn-products as it converges more quickly due to the eviction of stale embeddings. However, it takes longer to converge on ogbn-arxiv compared to other models. In contrast, VISAGNN maintains nearly the same memory usage as GAS and Refresh while accelerating the training process significantly (see Sec. \ref{sec:complexity} for complexity analysis). This improvement is attributed to VISAGNN’s ability to achieve the fastest convergence among all historical embedding baselines by reducing staleness, thereby requiring substantially fewer epochs to converge (see Sec. \ref{sec:convergence}). Moreover, while VISAGNN-Sum incurs slightly higher memory costs and running time than VISAGNN-Cat due to the inclusion of a fully connected layer, it demonstrates improved performance.

\subsection{Convergence Analysis}
\label{sec:convergence}
We provide a convergence analysis by comparing the test accuracy over time for baselines, including GAS, Refresh, and our proposed VISAGNN, on the ogbn-arxiv and ogbn-products datasets. The results in Figure \ref{fig:visa_arxiv} and \ref{fig:visa_products} (S stands for summation, C stands for concatenation) reveal that when staleness is not significant (as in the ogbn-arxiv case), Refresh performs poorly because it loses information from neighbors. However, when staleness is significant (as in the ogbn-products case), GAS's convergence is heavily affected by staleness. In contrast, our model achieves faster convergence and superior performance on both cases by effectively accounting for varying levels of staleness in the historical embeddings during training, as introduced in Sec. \ref{sec:method}. This advantage becomes especially clear on large datasets, where staleness tends to be more severe. Overall, these findings show that our algorithm not only improves performance but also accelerates convergence.

\begin{figure}[!ht]
  \begin{minipage}[t]{0.23\textwidth}
    \centering
    \hspace{-0.3in}
    \includegraphics[width=1\textwidth]{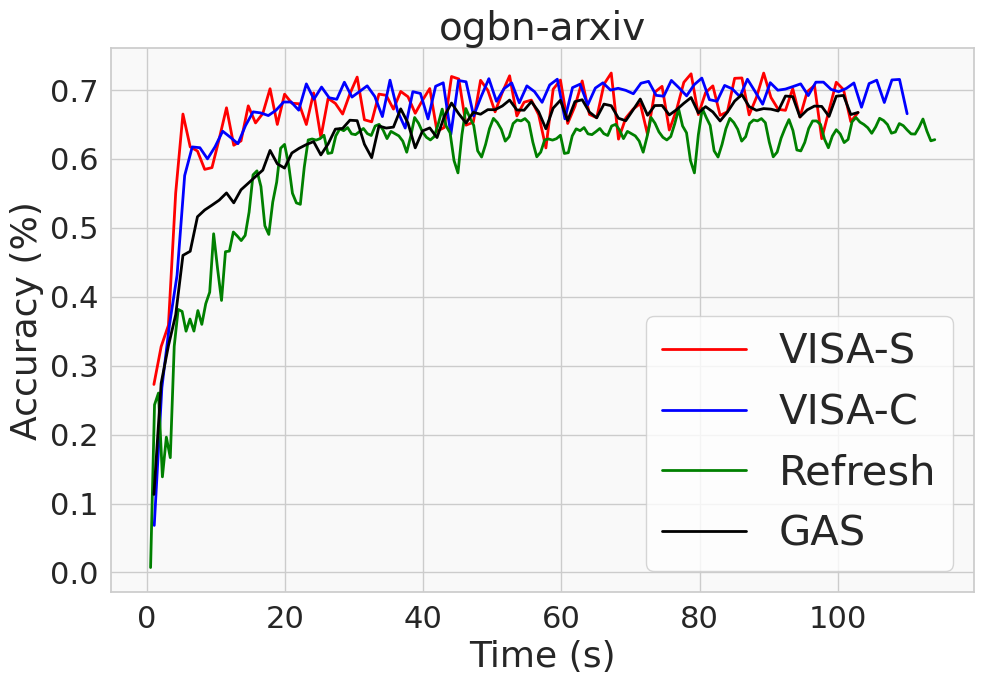}
    \caption{ogbn-arxiv}
    \label{fig:visa_arxiv}
  \end{minipage}
  \hspace{-5mm}
  \begin{minipage}[t]{0.23\textwidth}
    \centering
    \includegraphics[width=1\textwidth]{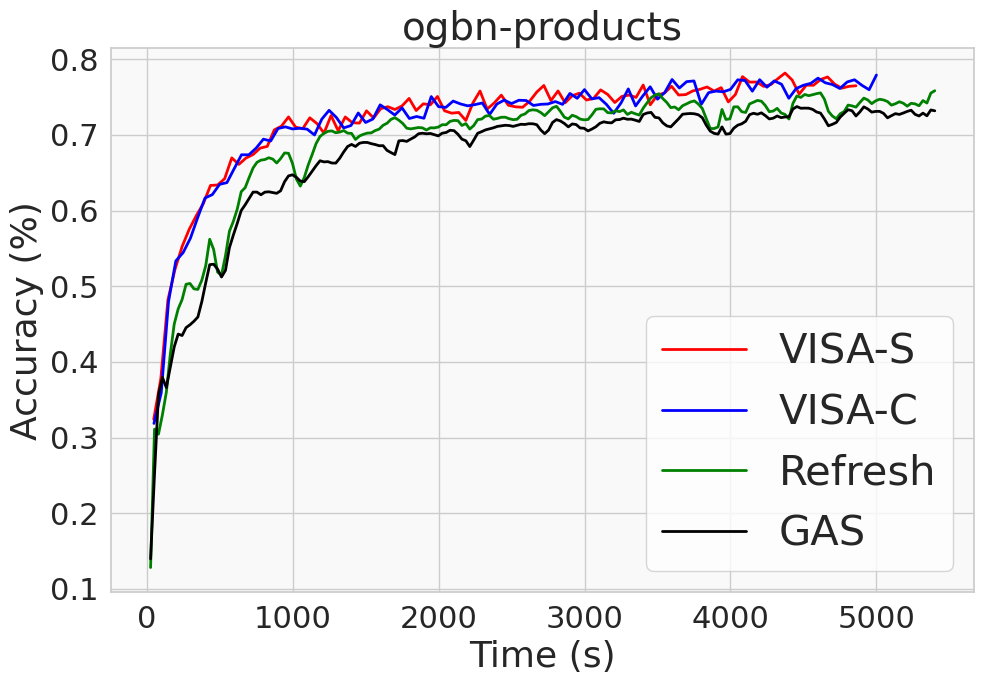}
    \caption{ogbn-products}
    \label{fig:visa_products}
  \end{minipage}
  \vspace{-3mm}
\end{figure}

\subsection{Ablation Study}
\subsubsection{Components}
Given the three novel techniques introduced in our VISAGNN, we conducted an ablation study on the ogbn-arxiv and ogbn-products datasets to identify which technique contributes the most to the final performance. For simplicity, we denote the dynamic attention, staleness-aware loss, and augmented embeddings as "att," "loss," and "emb," respectively. We use summation here for the augmented embedding. To ensure a fair comparison, all other hyperparameters are kept consistent, and we test various combinations of the three proposed strategies.

From Table~\ref{components}, we observe that the best performance occurs when all three techniques are applied. Specifically, the dynamic attention mechanism contributes the most, as it explicitly considers the staleness of each historical embedding during message passing and integrates this information into the training process, preventing overly stale embeddings from harming the final representation. Additionally, the proposed loss term enhances model performance by accounting for staleness in each training iteration, injecting this information into the gradient through backpropagation, thereby promoting staleness awareness in the model.

\begin{table}[h]
\centering
\caption{Accuracy ($\%$) for different batch sizes.}
\vspace{-3mm}
\label{tab:BS}
\resizebox{\linewidth}{!}{%
\begin{tabular}{l|c|c|ccccc}
\toprule
{{\textsc{Dataset}}} &{{\textsc{Clusters}}} & {{\textsc{BS}}} & {\textsc{GAS}} & {\textsc{FM}} & \textsc{LMC}& \textsc{\textbf{VISAGNN}} \\
\midrule
\multirow{3}{*}{\textbf{Products}} 
& \multirow{3}{*}{\textbf{150}} & 5   &  74.5 $\pm$ 0.6 & 74.8 $\pm$ 0.4 & 75.0 $\pm$ 0.4 & 77.1 $\pm$ 0.3  \\   
& & 10   & 75.6 $\pm$ 0.4 & 76.0 $\pm$ 0.3 & 76.3 $\pm$ 0.2 & 79.2 $\pm$ 0.3 \\  
&  & 20   & 77.0 $\pm$ 0.3 & 77.2 $\pm$ 0.2   & 77.5 $\pm$ 0.3 & 80.2 $\pm$ 0.2 \\
\midrule
\multirow{3}{*}{\textbf{Reddit}} 
& \multirow{3}{*}{\textbf{200}} & 20   & 94.8 $\pm$ 0.2 & 94.7 $\pm$ 0.3 & 95.0 $\pm$ 0.1 & 95.7 $\pm$ 0.1 \\
& & 50   & 95.0 $\pm$ 0.2  & 95.1 $\pm$ 0.3 & 95.7 $\pm$ 0.2 & 96.2 $\pm$ 0.1 \\
&  & 100  & 95.7 $\pm$ 0.1 & 95.6 $\pm$ 0.2   & 96.2 $\pm$ 0.1 & 96.6 $\pm$ 0.2\\
\midrule
\multirow{3}{*}{\textbf{Arxiv}} 
& \multirow{3}{*}{\textbf{40}} & 5   & 69.5 $\pm$ 0.4 & 70.1 $\pm$ 0.3 & 71.5 $\pm$ 0.2 & 72.7 $\pm$ 0.2 \\
&  & 10   & 70.1 $\pm$ 0.3 & 70.5 $\pm$ 0.3 & 71.8 $\pm$ 0.2 & 72.9 $\pm$ 0.2\\
&  & 20   & 71.7 $\pm$ 0.2 & 71.9 $\pm$ 0.2 & 72.2 $\pm$ 0.1 & 73.2 $\pm$ 0.2 \\
\bottomrule
\end{tabular}
}
\vspace{-3mm}
\end{table}

\subsubsection{Staleness resistance}
Previous results demonstrated that our model effectively mitigates the negative impact of stale embeddings. In this section, we further demonstrate our model's effectiveness in mitigating the negative effects of staleness by conducting experiments with varying levels of staleness through different batch sizes on ogbn-arxiv and ogbn-products. As mentioned in Sec.\ref{sec:intro}, when the batch size is small, the staleness becomes significant because there are more parameter updates within an epoch, while the historical embeddings are updated only once. We compare our model with existing representative historical embedding methods: GAS, GraphFM, and LMC. Note that the original Refresh method uses neighbor sampling. To ensure consistency of sampling across all models, we have excluded Refresh from this ablation study for a fair comparison. The results are presented in Table~\ref{tab:BS}. In these experiments, we strictly adhere to the settings outlined in their papers and official repositories. The graphs undergo pre-clustering using METIS~\citep{fey2021gnnautoscale}, with the total number of clusters detailed in Table under the label “Clusters.” The term "BS" refers to the number of clusters in the current mini-batch.

We observe that the performance of all baselines significantly drops as staleness increases. Specifically, we observe that GAS performs worse under high staleness, since it directly uses historical embeddings to approximate the true embeddings without any compensation. In contrast, our algorithms maintain strong performance across all cases, significantly outperforming all baselines, particularly in scenarios with large datasets and small batch sizes where staleness is prominent. For instance, with a batch size of 5, \textbf{we achieve a 2.6\% improvement and a 3.2\% improvement over GAS on the ogbn-products and ogbn-arxiv datasets, respectively.} This demonstrates the strong staleness resistance of our model, as all three proposed strategies help make the training process staleness-aware, effectively mitigating the negative impact of stale embeddings on model performance.

\begin{table}
    \centering 
    \renewcommand{\arraystretch}{1.2}
    \caption{Impact of hyperparameter $\lambda$}
    \vspace{-3mm}
    \label{tab:lambda} 
    \resizebox{0.75\linewidth}{!}{
        \begin{tabular}{l|cc}
        \toprule
        \textbf{Method} & \textbf{ogbn-arxiv} & \textbf{ogbn-products}\\
        \midrule
        $\lambda$ = 0 & 72.6$\pm$0.3 & 79.2$\pm$0.1 \\
        $\lambda$ = 0.3 & \textbf{73.2$\pm$0.2} & 79.6$\pm$0.1 \\
        $\lambda$ = 0.5 & 73.1$\pm$0.3 & \textbf{80.2$\pm$0.2} \\
        $\lambda$ = 0.8 & 72.3$\pm$0.3 & 78.8$\pm$0.2 \\
        \bottomrule
        \end{tabular}
    }
    \vspace{-3mm}
\end{table}

\subsubsection{Staleness-weighting $\lambda$}
In this section, we analyze the impact of $\lambda$, as it directly controls the extent to which staleness influences the loss and total gradient computation. We conduct experiments on ogbn-arxiv and ogbn-products with varying $\lambda$ values. From the results in Table~\ref{tab:lambda}, we observe that introducing the staleness loss with a moderate $\lambda$ (e.g., $\lambda=0.5$) achieves strong and robust performance across datasets. However, assigning too large a weight to the staleness loss can degrade performance, as it causes the model to overemphasize mitigating staleness at the expense of the primary learning task.

\begin{table}
    \centering 
    \caption{Performance Comparison with Additional Baselines.}
    \vspace{-3mm}
    \label{tab:performance_minor} 
    \setlength{\tabcolsep}{2pt}
    \resizebox{0.7\linewidth}{!}{
        \begin{tabular}{l|cc} 
        \toprule
        \textbf{Models} & \textsc{ogbn-arxiv} & \textsc{ogbn-products} \\
        \midrule
        SANCUS & 66.6$\pm$0.3 & 78.8$\pm$0.3 \\
        S3 & 72.5$\pm$0.3 & 76.9$\pm$0.3 \\
        SAT & 71.7$\pm$0.2 & 79.1$\pm$0.3 \\
        VISAGNN & \textbf{73.2$\pm$0.2} & \textbf{80.2$\pm$0.2} \\
        \bottomrule
        \end{tabular}
    }
    \vspace{-3mm}
\end{table}

\subsubsection{Additional Baselines}
In addition to the most relevant and widely cited baselines presented in Table~\ref{tab:baseline}, we further compare our approach with several recent methods~\cite{wang2024stalenessbased,peng2022sancus,bai2023staleness} that address the staleness problem from other perspectives, such as distributed learning, which differs from both our approach and the major baselines used in Table \ref{tab:baseline}. As shown in Table \ref{tab:performance_minor}, VISAGNN outperforms all state-of-the-art baselines. Note that our model is optimized from an algorithmic perspective and is general enough to be seamlessly integrated with these baselines.

\section{Conclusion}
Historical embedding methods have emerged as a promising solution for training GNNs on large-scale graphs by solving the neighbor explosion problem while maintaining model effectiveness. However, staleness has become a major limitation of these methods. In this work, we first present a theoretical analysis of this issue and then introduce VISAGNN, a versatile GNN framework that dynamically incorporates staleness criteria into the training process through three key designs. Experimental results show significant improvements over traditional historical embedding methods, particularly in scenarios with pronounced staleness, while accelerating model convergence and preserving good memory efficiency. It provides a flexible and efficient solution for large-scale GNN training.

\newpage
\bibliographystyle{ACM-Reference-Format}
\bibliography{ref_2026}


\begin{thebibliography}{46}


\ifx \showCODEN    \undefined \def \showCODEN     #1{\unskip}     \fi
\ifx \showISBNx    \undefined \def \showISBNx     #1{\unskip}     \fi
\ifx \showISBNxiii \undefined \def \showISBNxiii  #1{\unskip}     \fi
\ifx \showISSN     \undefined \def \showISSN      #1{\unskip}     \fi
\ifx \showLCCN     \undefined \def \showLCCN      #1{\unskip}     \fi
\ifx \shownote     \undefined \def \shownote      #1{#1}          \fi
\ifx \showarticletitle \undefined \def \showarticletitle #1{#1}   \fi
\ifx \showURL      \undefined \def \showURL       {\relax}        \fi
\providecommand\bibfield[2]{#2}
\providecommand\bibinfo[2]{#2}
\providecommand\natexlab[1]{#1}
\providecommand\showeprint[2][]{arXiv:#2}

\bibitem[Bai et~al\mbox{.}(2023)]%
        {bai2023staleness}
\bibfield{author}{\bibinfo{person}{Guangji Bai}, \bibinfo{person}{Ziyang Yu}, \bibinfo{person}{Zheng Chai}, \bibinfo{person}{Yue Cheng}, {and} \bibinfo{person}{Liang Zhao}.} \bibinfo{year}{2023}\natexlab{}.
\newblock \showarticletitle{Staleness-Alleviated Distributed GNN Training via Online Dynamic-Embedding Prediction}.
\newblock \bibinfo{journal}{\emph{arXiv preprint arXiv:2308.13466}} (\bibinfo{year}{2023}).
\newblock


\bibitem[Chai et~al\mbox{.}(2022)]%
        {chai2022distributed}
\bibfield{author}{\bibinfo{person}{Zheng Chai}, \bibinfo{person}{Guangji Bai}, \bibinfo{person}{Liang Zhao}, {and} \bibinfo{person}{Yue Cheng}.} \bibinfo{year}{2022}\natexlab{}.
\newblock \showarticletitle{Distributed Graph Neural Network Training with Periodic Historical Embedding Synchronization}.
\newblock \bibinfo{journal}{\emph{arXiv preprint arXiv:2206.00057}} (\bibinfo{year}{2022}).
\newblock


\bibitem[Chen et~al\mbox{.}(2018)]%
        {chen2018fastgcn}
\bibfield{author}{\bibinfo{person}{Jie Chen}, \bibinfo{person}{Tengfei Ma}, {and} \bibinfo{person}{Cao Xiao}.} \bibinfo{year}{2018}\natexlab{}.
\newblock \showarticletitle{Fastgcn: fast learning with graph convolutional networks via importance sampling}.
\newblock \bibinfo{journal}{\emph{arXiv preprint arXiv:1801.10247}} (\bibinfo{year}{2018}).
\newblock


\bibitem[Chen et~al\mbox{.}(2017)]%
        {chen2017stochastic}
\bibfield{author}{\bibinfo{person}{Jianfei Chen}, \bibinfo{person}{Jun Zhu}, {and} \bibinfo{person}{Le Song}.} \bibinfo{year}{2017}\natexlab{}.
\newblock \showarticletitle{Stochastic training of graph convolutional networks with variance reduction}.
\newblock \bibinfo{journal}{\emph{arXiv preprint arXiv:1710.10568}} (\bibinfo{year}{2017}).
\newblock


\bibitem[Chen et~al\mbox{.}(2020b)]%
        {chen2020simple}
\bibfield{author}{\bibinfo{person}{Ming Chen}, \bibinfo{person}{Zhewei Wei}, \bibinfo{person}{Zengfeng Huang}, \bibinfo{person}{Bolin Ding}, {and} \bibinfo{person}{Yaliang Li}.} \bibinfo{year}{2020}\natexlab{b}.
\newblock \showarticletitle{Simple and deep graph convolutional networks}. In \bibinfo{booktitle}{\emph{International Conference on Machine Learning}}. PMLR, \bibinfo{pages}{1725--1735}.
\newblock


\bibitem[Chen et~al\mbox{.}(2020a)]%
        {chen2020graph}
\bibfield{author}{\bibinfo{person}{Siheng Chen}, \bibinfo{person}{Yonina~C. Eldar}, {and} \bibinfo{person}{Lingxiao Zhao}.} \bibinfo{year}{2020}\natexlab{a}.
\newblock \bibinfo{title}{Graph Unrolling Networks: Interpretable Neural Networks for Graph Signal Denoising}.
\newblock
\showeprint[arxiv]{2006.01301}~[eess.SP]


\bibitem[Chiang et~al\mbox{.}(2019)]%
        {chiang2019cluster}
\bibfield{author}{\bibinfo{person}{Wei-Lin Chiang}, \bibinfo{person}{Xuanqing Liu}, \bibinfo{person}{Si Si}, \bibinfo{person}{Yang Li}, \bibinfo{person}{Samy Bengio}, {and} \bibinfo{person}{Cho-Jui Hsieh}.} \bibinfo{year}{2019}\natexlab{}.
\newblock \showarticletitle{Cluster-gcn: An efficient algorithm for training deep and large graph convolutional networks}. In \bibinfo{booktitle}{\emph{Proceedings of the 25th ACM SIGKDD international conference on knowledge discovery \& data mining}}. \bibinfo{pages}{257--266}.
\newblock


\bibitem[Cong et~al\mbox{.}(2020)]%
        {cong2020minimal}
\bibfield{author}{\bibinfo{person}{Weilin Cong}, \bibinfo{person}{Rana Forsati}, \bibinfo{person}{Mahmut Kandemir}, {and} \bibinfo{person}{Mehrdad Mahdavi}.} \bibinfo{year}{2020}\natexlab{}.
\newblock \showarticletitle{Minimal variance sampling with provable guarantees for fast training of graph neural networks}. In \bibinfo{booktitle}{\emph{Proceedings of the 26th ACM SIGKDD International Conference on Knowledge Discovery \& Data Mining}}. \bibinfo{pages}{1393--1403}.
\newblock


\bibitem[Fey et~al\mbox{.}(2021)]%
        {fey2021gnnautoscale}
\bibfield{author}{\bibinfo{person}{Matthias Fey}, \bibinfo{person}{Jan~E Lenssen}, \bibinfo{person}{Frank Weichert}, {and} \bibinfo{person}{Jure Leskovec}.} \bibinfo{year}{2021}\natexlab{}.
\newblock \showarticletitle{Gnnautoscale: Scalable and expressive graph neural networks via historical embeddings}. In \bibinfo{booktitle}{\emph{International Conference on Machine Learning}}. PMLR, \bibinfo{pages}{3294--3304}.
\newblock


\bibitem[Fout et~al\mbox{.}(2017)]%
        {fout2017protein}
\bibfield{author}{\bibinfo{person}{Alex Fout}, \bibinfo{person}{Jonathon Byrd}, \bibinfo{person}{Basir Shariat}, {and} \bibinfo{person}{Asa Ben-Hur}.} \bibinfo{year}{2017}\natexlab{}.
\newblock \showarticletitle{Protein interface prediction using graph convolutional networks}.
\newblock \bibinfo{journal}{\emph{Advances in neural information processing systems}}  \bibinfo{volume}{30} (\bibinfo{year}{2017}).
\newblock


\bibitem[Gasteiger et~al\mbox{.}(2018)]%
        {gasteiger2018predict}
\bibfield{author}{\bibinfo{person}{Johannes Gasteiger}, \bibinfo{person}{Aleksandar Bojchevski}, {and} \bibinfo{person}{Stephan G{\"u}nnemann}.} \bibinfo{year}{2018}\natexlab{}.
\newblock \showarticletitle{Predict then propagate: Graph neural networks meet personalized pagerank}.
\newblock \bibinfo{journal}{\emph{arXiv preprint arXiv:1810.05997}} (\bibinfo{year}{2018}).
\newblock


\bibitem[Gasteiger et~al\mbox{.}(2019)]%
        {gasteiger2018combining}
\bibfield{author}{\bibinfo{person}{Johannes Gasteiger}, \bibinfo{person}{Aleksandar Bojchevski}, {and} \bibinfo{person}{Stephan Günnemann}.} \bibinfo{year}{2019}\natexlab{}.
\newblock \showarticletitle{Combining Neural Networks with Personalized PageRank for Classification on Graphs}. In \bibinfo{booktitle}{\emph{International Conference on Learning Representations}}.
\newblock
\urldef\tempurl%
\url{https://openreview.net/forum?id=H1gL-2A9Ym}
\showURL{%
\tempurl}


\bibitem[Gu et~al\mbox{.}(2020)]%
        {gu2020implicit}
\bibfield{author}{\bibinfo{person}{Fangda Gu}, \bibinfo{person}{Heng Chang}, \bibinfo{person}{Wenwu Zhu}, \bibinfo{person}{Somayeh Sojoudi}, {and} \bibinfo{person}{Laurent El~Ghaoui}.} \bibinfo{year}{2020}\natexlab{}.
\newblock \showarticletitle{Implicit graph neural networks}.
\newblock \bibinfo{journal}{\emph{Advances in Neural Information Processing Systems}}  \bibinfo{volume}{33} (\bibinfo{year}{2020}), \bibinfo{pages}{11984--11995}.
\newblock


\bibitem[Hamilton et~al\mbox{.}(2017)]%
        {hamilton2017inductive}
\bibfield{author}{\bibinfo{person}{Will Hamilton}, \bibinfo{person}{Zhitao Ying}, {and} \bibinfo{person}{Jure Leskovec}.} \bibinfo{year}{2017}\natexlab{}.
\newblock \showarticletitle{Inductive representation learning on large graphs}.
\newblock \bibinfo{journal}{\emph{Advances in neural information processing systems}}  \bibinfo{volume}{30} (\bibinfo{year}{2017}).
\newblock


\bibitem[Hamilton(2020)]%
        {hamilton2020graph}
\bibfield{author}{\bibinfo{person}{William~L Hamilton}.} \bibinfo{year}{2020}\natexlab{}.
\newblock \showarticletitle{Graph representation learning}.
\newblock \bibinfo{journal}{\emph{Synthesis Lectures on Artifical Intelligence and Machine Learning}} \bibinfo{volume}{14}, \bibinfo{number}{3} (\bibinfo{year}{2020}), \bibinfo{pages}{1--159}.
\newblock


\bibitem[Han et~al\mbox{.}(2023)]%
        {han2023mlpinit}
\bibfield{author}{\bibinfo{person}{Xiaotian Han}, \bibinfo{person}{Tong Zhao}, \bibinfo{person}{Yozen Liu}, \bibinfo{person}{Xia Hu}, {and} \bibinfo{person}{Neil Shah}.} \bibinfo{year}{2023}\natexlab{}.
\newblock \showarticletitle{Mlpinit: Embarrassingly simple gnn training acceleration with mlp initialization}.
\newblock \bibinfo{journal}{\emph{ICLR}} (\bibinfo{year}{2023}).
\newblock


\bibitem[Hu et~al\mbox{.}(2020)]%
        {hu2020open}
\bibfield{author}{\bibinfo{person}{Weihua Hu}, \bibinfo{person}{Matthias Fey}, \bibinfo{person}{Marinka Zitnik}, \bibinfo{person}{Yuxiao Dong}, \bibinfo{person}{Hongyu Ren}, \bibinfo{person}{Bowen Liu}, \bibinfo{person}{Michele Catasta}, {and} \bibinfo{person}{Jure Leskovec}.} \bibinfo{year}{2020}\natexlab{}.
\newblock \showarticletitle{Open graph benchmark: Datasets for machine learning on graphs}.
\newblock \bibinfo{journal}{\emph{Advances in neural information processing systems}}  \bibinfo{volume}{33} (\bibinfo{year}{2020}), \bibinfo{pages}{22118--22133}.
\newblock


\bibitem[Huang et~al\mbox{.}(2023)]%
        {huang2023refresh}
\bibfield{author}{\bibinfo{person}{Kezhao Huang}, \bibinfo{person}{Haitian Jiang}, \bibinfo{person}{Minjie Wang}, \bibinfo{person}{Guangxuan Xiao}, \bibinfo{person}{David Wipf}, \bibinfo{person}{Xiang Song}, \bibinfo{person}{Quan Gan}, \bibinfo{person}{Zengfeng Huang}, \bibinfo{person}{Jidong Zhai}, {and} \bibinfo{person}{Zheng Zhang}.} \bibinfo{year}{2023}\natexlab{}.
\newblock \showarticletitle{ReFresh: Reducing Memory Access from Exploiting Stable Historical Embeddings for Graph Neural Network Training}.
\newblock \bibinfo{journal}{\emph{arXiv preprint arXiv:2301.07482}} (\bibinfo{year}{2023}).
\newblock


\bibitem[Huang et~al\mbox{.}(2020)]%
        {huang2020combining}
\bibfield{author}{\bibinfo{person}{Qian Huang}, \bibinfo{person}{Horace He}, \bibinfo{person}{Abhay Singh}, \bibinfo{person}{Ser-Nam Lim}, {and} \bibinfo{person}{Austin~R Benson}.} \bibinfo{year}{2020}\natexlab{}.
\newblock \showarticletitle{Combining label propagation and simple models out-performs graph neural networks}.
\newblock \bibinfo{journal}{\emph{arXiv preprint arXiv:2010.13993}} (\bibinfo{year}{2020}).
\newblock


\bibitem[Huang et~al\mbox{.}(2018)]%
        {huang2018adaptive}
\bibfield{author}{\bibinfo{person}{Wenbing Huang}, \bibinfo{person}{Tong Zhang}, \bibinfo{person}{Yu Rong}, {and} \bibinfo{person}{Junzhou Huang}.} \bibinfo{year}{2018}\natexlab{}.
\newblock \showarticletitle{Adaptive sampling towards fast graph representation learning}.
\newblock \bibinfo{journal}{\emph{Advances in neural information processing systems}}  \bibinfo{volume}{31} (\bibinfo{year}{2018}).
\newblock


\bibitem[Kipf and Welling(2016)]%
        {kipf2016semi}
\bibfield{author}{\bibinfo{person}{Thomas~N Kipf} {and} \bibinfo{person}{Max Welling}.} \bibinfo{year}{2016}\natexlab{}.
\newblock \showarticletitle{Semi-supervised classification with graph convolutional networks}.
\newblock \bibinfo{journal}{\emph{arXiv preprint arXiv:1609.02907}} (\bibinfo{year}{2016}).
\newblock


\bibitem[Li et~al\mbox{.}(2021)]%
        {li2021training}
\bibfield{author}{\bibinfo{person}{Guohao Li}, \bibinfo{person}{Matthias M{\"u}ller}, \bibinfo{person}{Bernard Ghanem}, {and} \bibinfo{person}{Vladlen Koltun}.} \bibinfo{year}{2021}\natexlab{}.
\newblock \showarticletitle{Training graph neural networks with 1000 layers}. In \bibinfo{booktitle}{\emph{International conference on machine learning}}. PMLR, \bibinfo{pages}{6437--6449}.
\newblock


\bibitem[Liu et~al\mbox{.}(2020)]%
        {liu2020towards}
\bibfield{author}{\bibinfo{person}{Meng Liu}, \bibinfo{person}{Hongyang Gao}, {and} \bibinfo{person}{Shuiwang Ji}.} \bibinfo{year}{2020}\natexlab{}.
\newblock \showarticletitle{Towards deeper graph neural networks}. In \bibinfo{booktitle}{\emph{Proceedings of the 26th ACM SIGKDD international conference on knowledge discovery \& data mining}}. \bibinfo{pages}{338--348}.
\newblock


\bibitem[Ma et~al\mbox{.}(2020)]%
        {ma2020unified}
\bibfield{author}{\bibinfo{person}{Yao Ma}, \bibinfo{person}{Xiaorui Liu}, \bibinfo{person}{Tong Zhao}, \bibinfo{person}{Yozen Liu}, \bibinfo{person}{Jiliang Tang}, {and} \bibinfo{person}{Neil Shah}.} \bibinfo{year}{2020}\natexlab{}.
\newblock \showarticletitle{A unified view on graph neural networks as graph signal denoising}.
\newblock \bibinfo{journal}{\emph{arXiv preprint arXiv:2010.01777}} (\bibinfo{year}{2020}).
\newblock


\bibitem[Ma and Tang(2021)]%
        {ma2021deep}
\bibfield{author}{\bibinfo{person}{Yao Ma} {and} \bibinfo{person}{Jiliang Tang}.} \bibinfo{year}{2021}\natexlab{}.
\newblock \bibinfo{booktitle}{\emph{Deep learning on graphs}}.
\newblock \bibinfo{publisher}{Cambridge University Press}.
\newblock


\bibitem[Pan et~al\mbox{.}(2020)]%
        {pan2020_unified}
\bibfield{author}{\bibinfo{person}{Xuran Pan}, \bibinfo{person}{Song Shiji}, {and} \bibinfo{person}{Huang Gao}.} \bibinfo{year}{2020}\natexlab{}.
\newblock \showarticletitle{A Unified Framework for Convolution-based Graph Neural Networks}.
\newblock \bibinfo{journal}{\emph{https://openreview.net/forum?id=zUMD--Fb9Bt}} (\bibinfo{year}{2020}).
\newblock


\bibitem[Peng et~al\mbox{.}(2022)]%
        {peng2022sancus}
\bibfield{author}{\bibinfo{person}{Jingshu Peng}, \bibinfo{person}{Zhao Chen}, \bibinfo{person}{Yingxia Shao}, \bibinfo{person}{Yanyan Shen}, \bibinfo{person}{Lei Chen}, {and} \bibinfo{person}{Jiannong Cao}.} \bibinfo{year}{2022}\natexlab{}.
\newblock \showarticletitle{Sancus: sta le n ess-aware c omm u nication-avoiding full-graph decentralized training in large-scale graph neural networks}.
\newblock \bibinfo{journal}{\emph{Proceedings of the VLDB Endowment}} \bibinfo{volume}{15}, \bibinfo{number}{9} (\bibinfo{year}{2022}), \bibinfo{pages}{1937--1950}.
\newblock


\bibitem[Rossi et~al\mbox{.}(2020)]%
        {rossi2020sign}
\bibfield{author}{\bibinfo{person}{Emanuele Rossi}, \bibinfo{person}{Fabrizio Frasca}, \bibinfo{person}{Ben Chamberlain}, \bibinfo{person}{Davide Eynard}, \bibinfo{person}{Michael Bronstein}, {and} \bibinfo{person}{Federico Monti}.} \bibinfo{year}{2020}\natexlab{}.
\newblock \showarticletitle{Sign: Scalable inception graph neural networks}.
\newblock \bibinfo{journal}{\emph{arXiv preprint arXiv:2004.11198}}  \bibinfo{volume}{7} (\bibinfo{year}{2020}), \bibinfo{pages}{15}.
\newblock


\bibitem[Sankar et~al\mbox{.}(2021)]%
        {sankar2021graph}
\bibfield{author}{\bibinfo{person}{Aravind Sankar}, \bibinfo{person}{Yozen Liu}, \bibinfo{person}{Jun Yu}, {and} \bibinfo{person}{Neil Shah}.} \bibinfo{year}{2021}\natexlab{}.
\newblock \showarticletitle{Graph neural networks for friend ranking in large-scale social platforms}. In \bibinfo{booktitle}{\emph{Proceedings of the Web Conference 2021}}. \bibinfo{pages}{2535--2546}.
\newblock


\bibitem[Shao et~al\mbox{.}(2022)]%
        {shao2022distributed}
\bibfield{author}{\bibinfo{person}{Yingxia Shao}, \bibinfo{person}{Hongzheng Li}, \bibinfo{person}{Xizhi Gu}, \bibinfo{person}{Hongbo Yin}, \bibinfo{person}{Yawen Li}, \bibinfo{person}{Xupeng Miao}, \bibinfo{person}{Wentao Zhang}, \bibinfo{person}{Bin Cui}, {and} \bibinfo{person}{Lei Chen}.} \bibinfo{year}{2022}\natexlab{}.
\newblock \showarticletitle{Distributed Graph Neural Network Training: A Survey}.
\newblock \bibinfo{journal}{\emph{arXiv preprint arXiv:2211.00216}} (\bibinfo{year}{2022}).
\newblock


\bibitem[Shi et~al\mbox{.}(2023)]%
        {shi2023lmc}
\bibfield{author}{\bibinfo{person}{Zhihao Shi}, \bibinfo{person}{Xize Liang}, {and} \bibinfo{person}{Jie Wang}.} \bibinfo{year}{2023}\natexlab{}.
\newblock \showarticletitle{LMC: Fast Training of GNNs via Subgraph Sampling with Provable Convergence}.
\newblock \bibinfo{journal}{\emph{arXiv preprint arXiv:2302.00924}} (\bibinfo{year}{2023}).
\newblock


\bibitem[Sun et~al\mbox{.}(2021)]%
        {sun2021scalable}
\bibfield{author}{\bibinfo{person}{Chuxiong Sun}, \bibinfo{person}{Hongming Gu}, {and} \bibinfo{person}{Jie Hu}.} \bibinfo{year}{2021}\natexlab{}.
\newblock \showarticletitle{Scalable and adaptive graph neural networks with self-label-enhanced training}.
\newblock \bibinfo{journal}{\emph{arXiv preprint arXiv:2104.09376}} (\bibinfo{year}{2021}).
\newblock


\bibitem[Tang et~al\mbox{.}(2020)]%
        {tang2020knowing}
\bibfield{author}{\bibinfo{person}{Xianfeng Tang}, \bibinfo{person}{Yozen Liu}, \bibinfo{person}{Neil Shah}, \bibinfo{person}{Xiaolin Shi}, \bibinfo{person}{Prasenjit Mitra}, {and} \bibinfo{person}{Suhang Wang}.} \bibinfo{year}{2020}\natexlab{}.
\newblock \showarticletitle{Knowing your fate: Friendship, action and temporal explanations for user engagement prediction on social apps}. In \bibinfo{booktitle}{\emph{Proceedings of the 26th ACM SIGKDD international conference on knowledge discovery \& data mining}}. \bibinfo{pages}{2269--2279}.
\newblock


\bibitem[Veli{\v{c}}kovi{\'c} et~al\mbox{.}(2017)]%
        {velivckovic2017graph}
\bibfield{author}{\bibinfo{person}{Petar Veli{\v{c}}kovi{\'c}}, \bibinfo{person}{Guillem Cucurull}, \bibinfo{person}{Arantxa Casanova}, \bibinfo{person}{Adriana Romero}, \bibinfo{person}{Pietro Lio}, {and} \bibinfo{person}{Yoshua Bengio}.} \bibinfo{year}{2017}\natexlab{}.
\newblock \showarticletitle{Graph attention networks}.
\newblock \bibinfo{journal}{\emph{arXiv preprint arXiv:1710.10903}} (\bibinfo{year}{2017}).
\newblock


\bibitem[Wang et~al\mbox{.}(2024)]%
        {wang2024stalenessbased}
\bibfield{author}{\bibinfo{person}{Limei Wang}, \bibinfo{person}{Si Zhang}, \bibinfo{person}{Hanqing Zeng}, \bibinfo{person}{Hao Wu}, \bibinfo{person}{Zhigang Hua}, \bibinfo{person}{Kaveh Hassani}, \bibinfo{person}{Andrey Malevich}, \bibinfo{person}{Bo Long}, {and} \bibinfo{person}{Shuiwang Ji}.} \bibinfo{year}{2024}\natexlab{}.
\newblock \bibinfo{title}{Staleness-based subgraph sampling for large-scale {GNN}s training}.
\newblock
\urldef\tempurl%
\url{https://openreview.net/forum?id=H7z1gHsaZ0}
\showURL{%
\tempurl}


\bibitem[Wu et~al\mbox{.}(2019)]%
        {wu2019simplifying}
\bibfield{author}{\bibinfo{person}{Felix Wu}, \bibinfo{person}{Amauri Souza}, \bibinfo{person}{Tianyi Zhang}, \bibinfo{person}{Christopher Fifty}, \bibinfo{person}{Tao Yu}, {and} \bibinfo{person}{Kilian Weinberger}.} \bibinfo{year}{2019}\natexlab{}.
\newblock \showarticletitle{Simplifying graph convolutional networks}. In \bibinfo{booktitle}{\emph{International conference on machine learning}}. PMLR, \bibinfo{pages}{6861--6871}.
\newblock


\bibitem[Wu et~al\mbox{.}(2022)]%
        {wu2022graph}
\bibfield{author}{\bibinfo{person}{Shiwen Wu}, \bibinfo{person}{Fei Sun}, \bibinfo{person}{Wentao Zhang}, \bibinfo{person}{Xu Xie}, {and} \bibinfo{person}{Bin Cui}.} \bibinfo{year}{2022}\natexlab{}.
\newblock \showarticletitle{Graph neural networks in recommender systems: a survey}.
\newblock \bibinfo{journal}{\emph{Comput. Surveys}} \bibinfo{volume}{55}, \bibinfo{number}{5} (\bibinfo{year}{2022}), \bibinfo{pages}{1--37}.
\newblock


\bibitem[Xue et~al\mbox{.}(2023a)]%
        {xue2023lazygnn}
\bibfield{author}{\bibinfo{person}{Rui Xue}, \bibinfo{person}{Haoyu Han}, \bibinfo{person}{MohamadAli Torkamani}, \bibinfo{person}{Jian Pei}, {and} \bibinfo{person}{Xiaorui Liu}.} \bibinfo{year}{2023}\natexlab{a}.
\newblock \showarticletitle{Lazygnn: Large-scale graph neural networks via lazy propagation}. In \bibinfo{booktitle}{\emph{International Conference on Machine Learning}}. PMLR, \bibinfo{pages}{38926--38937}.
\newblock


\bibitem[Xue et~al\mbox{.}(2023b)]%
        {xue2023efficient}
\bibfield{author}{\bibinfo{person}{Rui Xue}, \bibinfo{person}{Xipeng Shen}, \bibinfo{person}{Ruozhou Yu}, {and} \bibinfo{person}{Xiaorui Liu}.} \bibinfo{year}{2023}\natexlab{b}.
\newblock \showarticletitle{Efficient large language models fine-tuning on graphs}.
\newblock \bibinfo{journal}{\emph{arXiv preprint arXiv:2312.04737}} (\bibinfo{year}{2023}).
\newblock


\bibitem[Xue et~al\mbox{.}(2024)]%
        {xue2024haste}
\bibfield{author}{\bibinfo{person}{Rui Xue}, \bibinfo{person}{Tong Zhao}, \bibinfo{person}{Neil Shah}, {and} \bibinfo{person}{Xiaorui Liu}.} \bibinfo{year}{2024}\natexlab{}.
\newblock \showarticletitle{Haste makes waste: A simple approach for scaling graph neural networks}.
\newblock \bibinfo{journal}{\emph{arXiv preprint arXiv:2410.05416}} (\bibinfo{year}{2024}).
\newblock


\bibitem[Ying et~al\mbox{.}(2018)]%
        {ying2018graph}
\bibfield{author}{\bibinfo{person}{Rex Ying}, \bibinfo{person}{Ruining He}, \bibinfo{person}{Kaifeng Chen}, \bibinfo{person}{Pong Eksombatchai}, \bibinfo{person}{William~L Hamilton}, {and} \bibinfo{person}{Jure Leskovec}.} \bibinfo{year}{2018}\natexlab{}.
\newblock \showarticletitle{Graph convolutional neural networks for web-scale recommender systems}. In \bibinfo{booktitle}{\emph{Proceedings of the 24th ACM SIGKDD international conference on knowledge discovery \& data mining}}. \bibinfo{pages}{974--983}.
\newblock


\bibitem[Yu et~al\mbox{.}(2022)]%
        {yu2022graphfm}
\bibfield{author}{\bibinfo{person}{Haiyang Yu}, \bibinfo{person}{Limei Wang}, \bibinfo{person}{Bokun Wang}, \bibinfo{person}{Meng Liu}, \bibinfo{person}{Tianbao Yang}, {and} \bibinfo{person}{Shuiwang Ji}.} \bibinfo{year}{2022}\natexlab{}.
\newblock \showarticletitle{GraphFM: Improving large-scale GNN training via feature momentum}. In \bibinfo{booktitle}{\emph{International Conference on Machine Learning}}. PMLR, \bibinfo{pages}{25684--25701}.
\newblock


\bibitem[Zeng et~al\mbox{.}(2020)]%
        {Zeng2020GraphSAINT}
\bibfield{author}{\bibinfo{person}{Hanqing Zeng}, \bibinfo{person}{Hongkuan Zhou}, \bibinfo{person}{Ajitesh Srivastava}, \bibinfo{person}{Rajgopal Kannan}, {and} \bibinfo{person}{Viktor Prasanna}.} \bibinfo{year}{2020}\natexlab{}.
\newblock \showarticletitle{GraphSAINT: Graph Sampling Based Inductive Learning Method}. In \bibinfo{booktitle}{\emph{International Conference on Learning Representations}}.
\newblock
\urldef\tempurl%
\url{https://openreview.net/forum?id=BJe8pkHFwS}
\showURL{%
\tempurl}


\bibitem[Zhang et~al\mbox{.}(2024)]%
        {zhang2024linear}
\bibfield{author}{\bibinfo{person}{Jiahao Zhang}, \bibinfo{person}{Rui Xue}, \bibinfo{person}{Wenqi Fan}, \bibinfo{person}{Xin Xu}, \bibinfo{person}{Qing Li}, \bibinfo{person}{Jian Pei}, {and} \bibinfo{person}{Xiaorui Liu}.} \bibinfo{year}{2024}\natexlab{}.
\newblock \showarticletitle{Linear-time graph neural networks for scalable recommendations}. In \bibinfo{booktitle}{\emph{Proceedings of the ACM Web Conference 2024}}. \bibinfo{pages}{3533--3544}.
\newblock


\bibitem[Zhu et~al\mbox{.}(2021)]%
        {zhu2021interpreting}
\bibfield{author}{\bibinfo{person}{Meiqi Zhu}, \bibinfo{person}{Xiao Wang}, \bibinfo{person}{Chuan Shi}, \bibinfo{person}{Houye Ji}, {and} \bibinfo{person}{Peng Cui}.} \bibinfo{year}{2021}\natexlab{}.
\newblock \bibinfo{title}{Interpreting and Unifying Graph Neural Networks with An Optimization Framework}.
\newblock
\showeprint[arxiv]{2101.11859}~[cs.LG]


\bibitem[Zou et~al\mbox{.}(2019)]%
        {zou2019layer}
\bibfield{author}{\bibinfo{person}{Difan Zou}, \bibinfo{person}{Ziniu Hu}, \bibinfo{person}{Yewen Wang}, \bibinfo{person}{Song Jiang}, \bibinfo{person}{Yizhou Sun}, {and} \bibinfo{person}{Quanquan Gu}.} \bibinfo{year}{2019}\natexlab{}.
\newblock \showarticletitle{Layer-dependent importance sampling for training deep and large graph convolutional networks}.
\newblock \bibinfo{journal}{\emph{Advances in neural information processing systems}}  \bibinfo{volume}{32} (\bibinfo{year}{2019}).
\newblock


\end{thebibliography}

\newpage
\appendix
\newpage

\section{Proof of Theorem 1}
\label{app:proof1}

\begin{proof}  
Suppose $\Tilde{g}_\theta^{(l)}$ is a historical embedding-based GNN with L-layers, then the whole GNN model can be defined as $ \Tilde{h}^{(L)} = \Tilde{g}_\theta^{(L)}\circ\Tilde{g}_\theta^{(L-1)}\circ \dots \circ\Tilde{g}_\theta^{(1)}$, similarly, the full batch GNN can be defined as:
$h^{(L)} = g_\theta^{(L)}\circ g_\theta^{(L-1)}\circ \dots \circ g_\theta^{(1)}$, then:

\begin{equation}
||\Tilde{h}^{(L)} - h^{(L)}|| = ||\Tilde{g}_\theta^{(L)}\circ\Tilde{g}_\theta^{(L-1)}\circ \dots \circ\Tilde{g}_\theta^{(1)} - g_\theta^{(L)}\circ g_\theta^{(L-1)}\circ \dots \circ g_\theta^{(1)}||
\end{equation}
\begin{equation}
=||\Tilde{g}_\theta^{(L)}\circ\Tilde{g}_\theta^{(L-1)}\circ \dots \circ\Tilde{g}_\theta^{(1)} - \Tilde{g}_\theta^{(L)}\circ\Tilde{g}_\theta^{(L-1)}\circ \dots \circ g_\theta^{(1)} \nonumber
\end{equation}
\begin{equation}
+\Tilde{g}_\theta^{(L)}\circ\Tilde{g}_\theta^{(L-1)}\circ \dots  \circ\Tilde{g}_\theta^{(2)} \circ g_\theta^{(1)} - \Tilde{g}_\theta^{(L)}\circ\Tilde{g}_\theta^{(L-1)}\circ \dots \circ g_\theta^{(2)} \circ g_\theta^{(1)} - \dots \nonumber
\end{equation}
\begin{equation}
+ \Tilde{g}_\theta^{(L)}\circ g_\theta^{(L-1)}\circ \dots \circ g_\theta^{(1)} - g_\theta^{(L)}\circ g_\theta^{(L-1)}\circ \dots \circ g_\theta^{(1)}||
\end{equation}
\begin{equation}
\leq ||\Tilde{g}_\theta^{(L)}\circ\Tilde{g}_\theta^{(L-1)}\circ \dots \circ\Tilde{g}_\theta^{(1)} - \Tilde{g}_\theta^{(L)}\circ\Tilde{g}_\theta^{(L-1)}\circ \dots \circ g_\theta^{(1)}|| + 
\dots  \nonumber
\end{equation}
\begin{equation}
+ ||\Tilde{g}_\theta^{(L)}\circ g_\theta^{(L-1)}\circ \dots \circ g_\theta^{(1)} - g_\theta^{(L)}\circ g_\theta^{(L-1)}\circ \dots \circ g_\theta^{(1)}||\\
\end{equation}
\begin{equation}
=\sum_{k=1}^{L}\left(\prod_{l=k+1}^{L}\beta^{(l)}||\Tilde{g}_\theta^{(k)}\circ g_\theta^{(k-1)}\circ \dots \circ g_\theta^{(1)} - g_\theta^{(k)}\circ g_\theta^{(k-1)}\circ \dots \circ g_\theta^{(1)}||\right)
\end{equation}
\begin{equation}
=\sum_{k=1}^{L}\left(\prod_{l=k+1}^{L}\beta^{(l)}||g_{\theta}^{(k)}\left(h^{(k-1)}_{i}, \Bar{h}^{(k-1)}\right) - g_{\theta}^{(k)}\left(h^{(k-1)}_{i}, h^{(k-1)}\right)||\right)
\end{equation}
\begin{equation}
\leq \sum_{k=1}^{L}\left(\prod_{l=k+1}^{L}\beta^{(l)}||\sum_{\cN(i)}\Tilde{\hat{A}}_{i,} * \Bar{h}^{(k-1)} - \sum_{ \cN(i)}\Tilde{\hat{A}}_{i,} * h^{(k-1)}||\right)
\end{equation}
\begin{equation}
\leq \sum_{k=1}^{L}\left(\prod_{l=k+1}^{L}\beta^{(l)} |\cN(i)|*||\Tilde{\hat{A}}_{i,} * \Bar{h}^{(k-1)} - \Tilde{\hat{A}}_{i,} * h^{(k-1)}||\right)
\end{equation}
\begin{equation}
\leq \sum_{k=1}^{L}\left(\prod_{l=k+1}^{L}\beta^{(l)}  |\cN(i)|*||\Tilde{\hat{A}}_{i,}||*||\Bar{h}^{(k-1)}-h^{(k-1)}||\right)
\end{equation}
\end{proof}

\section{Proof of Theorem 2}
\label{app:proof2}

\begin{proof}
First verify Lipschitz continuity of the aggregation operator. For any two sets of embeddings $\{h_j\}$ and $\{h'_j\}$:
\begin{equation}
\|\tilde{h}_i^{(l)} - \tilde{h}_i'^{(l)}\| = \left\|\phi\left(\sum_{j} \alpha_{ij}^{(l)} W^{(l)} h_j^{(l-1)}\right) - \phi\left(\sum_{j} \alpha_{ij}'^{(l)} W^{(l)} h_j'^{(l-1)}\right)\right\|
\end{equation}

Using Lipschitz property of $\phi$ with constant $L_\phi$:
\begin{equation}
\leq L_\phi \left\|\sum_{j} \alpha_{ij}^{(l)} W^{(l)} h_j^{(l-1)} - \sum_{j} \alpha_{ij}'^{(l)} W^{(l)} h_j'^{(l-1)}\right\|
\end{equation}

Add and subtract intermediate term:
\begin{equation}
= L_\phi \left\|\sum_{j} \alpha_{ij}^{(l)} W^{(l)} (h_j^{(l-1)} - h_j'^{(l-1)}) + \sum_{j} (\alpha_{ij}^{(l)} - \alpha_{ij}'^{(l)}) W^{(l)} h_j'^{(l-1)}\right\|
\end{equation}
\begin{equation}
\leq L_\phi \|W^{(l)}\| \left(\underbrace{\sum_{j} \alpha_{ij}^{(l)} \|h_j^{(l-1)} - h_j'^{(l-1)}\|}_{\text{Term 1: fixed weights}} + \underbrace{\sum_{j} |\alpha_{ij}^{(l)} - \alpha_{ij}'^{(l)}| \|h_j'^{(l-1)}\|}_{\text{Term 2: weight changes}}\right) 
\end{equation}

\textbf{Bounding Term 1:} Since $\sum_j \alpha_{ij}^{(l)} = 1$ and $\alpha_{ij}^{(l)} \geq 0$:
\begin{equation}
\sum_{j} \alpha_{ij}^{(l)} \|h_j^{(l-1)} - h_j'^{(l-1)}\| \leq \max_{j \in \mathcal{N}(i)} \|h_j^{(l-1)} - h_j'^{(l-1)}\|
\end{equation}

\textbf{Bounding Term 2:}Recall:
\begin{equation}
\alpha_{ij}^{(l)} = \frac{\exp(f_{ij} - \gamma(t) s_j \sigma(c_j - c_{avg}))}{\sum_{k \in \mathcal{N}(i)} \exp(f_{ik} - \gamma(t) s_k \sigma(c_k - c_{avg}))}
\end{equation}
where $f_{ij} = \text{LeakyReLU}(a^T[Wh_i \| W\bar{h}_j])$.

\begin{lemma}[Staleness-Aware Attention Lipschitz Constant]
Suppose $\gamma_{\max} = \sup_t \gamma(t) = \gamma(0)$. The staleness-aware attention mechanism has Lipschitz constant $L_\alpha$ such that:
\begin{equation}
\|\alpha(h) - \alpha(h')\|_1 \leq L_\alpha \max_{j} \|h_j - h'_j\|
\end{equation}
where $L_\alpha = 2\|a\|\|W\|(1 + \gamma_{\max}\sigma_{\max})$ with $\sigma_{\max} = \max_{j} |\sigma(c_j - c_{avg})|$.
\end{lemma}

\begin{proof}[Proof of Lemma]
For the softmax function with additional staleness terms, the gradient satisfies:
\begin{equation}
\frac{\partial \alpha_{ij}}{\partial f_{ik}} = \alpha_{ij}(\delta_{jk} - \alpha_{ik})
\end{equation}

The gradient of the attention score with respect to embeddings is:
\begin{equation}
\left\|\frac{\partial f_{ij}}{\partial h_k}\right\| \leq \|a\|\|W\|
\end{equation}

Including the staleness penalty term:
\begin{equation}
\left\|\frac{\partial}{\partial h_k}[f_{ij} - \gamma(t) s_j \sigma(c_j - c_{avg})]\right\| \leq \|a\|\|W\|(1 + \gamma_{\max}\sigma_{\max})
\end{equation}

By the chain rule and properties of softmax:
\begin{equation}
\left\|\frac{\partial \alpha_{ij}}{\partial h_k}\right\| \leq \|a\|\|W\|(1 + \gamma_{\max}\sigma_{\max})|\alpha_{ij}(\delta_{jk} - \alpha_{ik})|
\end{equation}

Summing over all $i,j$ and using $\sum_j \alpha_{ij} = 1$:
\begin{equation}
\sum_{i,j} \left\|\frac{\partial \alpha_{ij}}{\partial h_k}\right\| \leq 2\|a\|\|W\|(1 + \gamma_{\max}\sigma_{\max})
\end{equation}

This gives the bound $L_\alpha = 2\|a\|\|W\|(1 + \gamma_{\max}\sigma_{\max})$.
\end{proof}

Therefore, Term 2 is bounded by:
\begin{equation}
\sum_{j} |\alpha_{ij}^{(l)} - \alpha_{ij}'^{(l)}| \|h_j'^{(l-1)}\| \leq L_\alpha H_{\max} \max_{j} \|h_j^{(l-1)} - h_j'^{(l-1)}\|
\end{equation}
where $H_{\max} = \max_{j} \|h_j\|$ is the maximum embedding norm.

\textbf{Final Lipschitz constant:} Combining both terms:
\begin{equation}
\|\tilde{h}_i^{(l)} - \tilde{h}_i'^{(l)}\| \leq \beta^{(l)} \max_{j \in \mathcal{N}(i)} \|h_j^{(l-1)} - h_j'^{(l-1)}\|
\end{equation}
where 
\begin{equation}
\beta^{(l)} = L_\phi \|W^{(l)}\| (1 + L_\alpha H_{\max}), \text{with} L_\alpha = 2\|a\|\|W\|(1 + \gamma_{\max}\sigma_{\max}).
\end{equation}

For the composite loss function $L(\theta) = L_{\text{task}}(\theta) + \lambda L_{\text{stale}}(\theta)$:

The task loss $L_{\text{task}}(\theta) = \ell(f_\theta(X, A))$ satisfies:
\begin{equation}
\|\nabla L_{\text{task}}(\theta) - \nabla L_{\text{task}}(\theta')\| \leq L_{\text{task}} \prod_{l=1}^L \beta^{(l)} \|\theta - \theta'\|
\end{equation}

The staleness regularization term $L_{\text{stale}} = \sum_{i \in \mathcal{B}} \|h_i^{(L)}(\theta_t) - h_i^{(L)}(\theta_{t-1})\|^2$ has Lipschitz constant 2 with respect to $h_i^{(L)}(\theta_t)$.

Therefore, the composite loss is L-smooth with:
\begin{equation}
L = \left(L_{\text{task}} + 2\lambda\right) \prod_{l=1}^L \beta^{(l)}
\end{equation}

Suppose $g_t$ is the stochastic gradient with $\mathbb{E}[g_t] = \nabla L(\theta_t)$ and $\mathbb{E}[\|g_t - \nabla L(\theta_t)\|^2] \leq \sigma^2$.

Using L-smoothness:
\begin{equation}
L(\theta_{t+1}) \leq L(\theta_t) + \langle \nabla L(\theta_t), \theta_{t+1} - \theta_t \rangle + \frac{L}{2}\|\theta_{t+1} - \theta_t\|^2
\end{equation}

Substituting $\theta_{t+1} - \theta_t = -\eta g_t$ and taking expectation over $g_t$ (conditioned on $\theta_t$):
\begin{equation}
\mathbb{E}[L(\theta_{t+1}) \mid \theta_t] \leq L(\theta_t) - \eta \langle \nabla L(\theta_t), \mathbb{E}[g_t \mid \theta_t] \rangle + \frac{\eta^2 L}{2}\mathbb{E}[\|g_t\|^2 \mid \theta_t]
\end{equation}

\begin{equation}
\mathbb{E}[L(\theta_{t+1}) \mid \theta_t] \leq L(\theta_t) - \eta \|\nabla L(\theta_t)\|^2 + \frac{\eta^2 L}{2}(\sigma^2 + \|\nabla L(\theta_t)\|^2)
\end{equation}

\begin{equation}
\mathbb{E}[L(\theta_{t+1}) \mid \theta_t] \leq L(\theta_t) - \eta\left(1 - \frac{\eta L}{2}\right) \|\nabla L(\theta_t)\|^2 + \frac{\eta^2 L \sigma^2}{2}
\end{equation}

Applying the law of total expectation and conditioning on $\theta_t$:
\begin{equation}
\mathbb{E}[L(\theta_{t+1})] \leq \mathbb{E}[L(\theta_t)] - \eta\left(1 - \frac{\eta L}{2}\right) \mathbb{E}[\|\nabla L(\theta_t)\|^2] + \frac{\eta^2 L \sigma^2}{2}
\end{equation}

For $\eta < \frac{2}{L}$, we have $1 - \frac{\eta L}{2} > 0$. Rearranging:

\begin{equation}
\mathbb{E}[\|\nabla L(\theta_t)\|^2] \leq \frac{2}{\eta(2 - \eta L)}(\mathbb{E}[L(\theta_t)] - \mathbb{E}[L(\theta_{t+1})]) + \frac{\eta L \sigma^2}{2 - \eta L}
\end{equation}

Summing over $t = 0, 1, \ldots, T-1$:
\begin{equation}
\sum_{t=0}^{T-1} \mathbb{E}[\|\nabla L(\theta_t)\|^2] \leq \frac{2}{\eta(2 - \eta L)}\sum_{t=0}^{T-1}(\mathbb{E}[L(\theta_t)] - \mathbb{E}[L(\theta_{t+1})]) + \frac{T \eta L \sigma^2}{2 - \eta L}
\end{equation}

Using telescoping property:
\begin{equation}
\sum_{t=0}^{T-1}(\mathbb{E}[L(\theta_t)] - \mathbb{E}[L(\theta_{t+1})]) = \mathbb{E}[L(\theta_0)] - \mathbb{E}[L(\theta_T)] \leq L(\theta_0) - L^*
\end{equation}

Dividing by $T$:
\begin{equation}
\frac{1}{T}\sum_{t=0}^{T-1} \mathbb{E}[\|\nabla L(\theta_t)\|^2] \leq \frac{2(L(\theta_0) - L^*)}{\eta T(2 - \eta L)} + \frac{\eta L \sigma^2}{2 - \eta L}
\end{equation}

This completes the proof with the correct handling of expectations and proper use of the bounded variance assumption.

\end{proof}

\end{document}